%% file: main-decentralized-mixing.tex
\documentclass{article}

\usepackage{microtype}
\usepackage{graphicx}
\usepackage{caption}
\usepackage{subcaption}
\usepackage{booktabs} 
\usepackage{hyperref}
\usepackage{url}
\usepackage{mathtools}
\usepackage{bm}
\usepackage{tikz}
\usetikzlibrary{shapes,fit,arrows,positioning}
\tikzstyle{vertex} = [circle, draw, thick, text centered,minimum size=1cm]
\tikzstyle{edge} = [draw, thick,<->]

\usepackage{hyperref}

\usepackage[accepted]{icml2022mod}

\providecommand{\uu}{\mathbf{u}}

\providecommand{\xx}{\mathbf{x}}
\providecommand{\yy}{\mathbf{y}}

\newcommand{\Ea}[1]{\E\left[#1\right]}

\DeclarePairedDelimiterX{\rbr}[1]{(}{)}{#1} 
\DeclarePairedDelimiterX{\sbr}[1]{[}{]}{#1}

\DeclarePairedDelimiterX{\cbr}[1]{\{}{\}}{#1}
\input{math_commands.tex}
\input{macros.tex}

\usepackage{placeins}
\usepackage{afterpage}

\usepackage[textwidth=5cm]{todonotes}
\usepackage{tabularx}
\commenter{Seb}{red}
\commenter{ak}{purple}
\commenter{yd}{green}
\commenter{mj}{blue}

\icmltitlerunning{Data-heterogeneity-aware Mixing for Decentralized Learning}

\begin{document}

\twocolumn[
\icmltitle{Data-heterogeneity-aware Mixing for Decentralized Learning}



\icmlsetsymbol{equal}{*}

\begin{icmlauthorlist}
\icmlauthor{Yatin Dandi}{iit,epfl}
\icmlauthor{Anastasia Koloskova}{epfl}
\icmlauthor{Martin Jaggi}{epfl}
\icmlauthor{Sebastian U. Stich}{cispa}
\end{icmlauthorlist}

\icmlaffiliation{iit}{IIT Kanpur, India}
\icmlaffiliation{epfl}{EPFL, Switzerland}
\icmlaffiliation{cispa}{CISPA Helmholtz Center for Information Security, Germany}


\icmlkeywords{Machine Learning}

\vskip 0.3in
]



\printAffiliationsAndNotice{}  


\begin{abstract}

Decentralized learning provides an effective framework to train machine learning models with data distributed over arbitrary communication graphs. However, most existing approaches towards decentralized learning disregard the interaction between data heterogeneity and graph topology. In this paper, we characterize the dependence of convergence on the relationship between the mixing weights of the graph and the data heterogeneity across nodes. We propose a metric that quantifies the ability of a graph to mix the current gradients. We further prove that the metric controls the convergence rate, particularly in settings where the heterogeneity across nodes dominates the stochasticity between updates for a given node. Motivated by our analysis, we propose an approach that periodically and efficiently optimizes the metric using standard convex constrained optimization and sketching techniques. Through comprehensive experiments on standard computer vision and NLP benchmarks, we show that our approach leads to improvement in test performance for a wide range of tasks.

\end{abstract}

\section{Introduction}
Machine learning is  gradually shifting  from classical centralized training to decentralized data processing. For example, federated learning (FL) allows multiple parties to jointly train an ML model together without disclosing their personal data to others~\cite{Kairouz2019:federated}.
While FL training relies on a central coordinator, 
fully distributed learning methods instead use direct peer-to-peer communication between the parties (e.g.\ personal devices, organization, or compute nodes inside a datacenter)~\cite{NIPS2017_lian,pmlr-v97-assran19a,KoloskovaLSJ19decentralized,nedic2020review}.
In decentralized learning, communication is limited to the network topology.
The nodes can only communicate with their direct neighbors in the network in each round of (one hop) communication~\cite{Tsitsiklis1985:gossip}.

While the underlying network topology is fixed---for instance implied by physical constraints---the nodes can freely choose with which neighbors they want to communicate. This results in a (possibly time-varying) communication graph that respects the underlying  network topology. Examples are
 applications such as  sensor networks~\cite{mihaylov2009decentralized}, multi-agent robotic systems~\cite{long2018towards}, IoT systems~\cite{wang2020federated}, edge devices connected over wireless networks or the internet~\cite{Kairouz2019:federated}, and nodes connected in a datacenter~\cite{pmlr-v97-assran19a}. \looseness=-1 

Decentralized learning with distributed SGD~\citep[D-SGD,][]{NIPS2017_lian} faces at least two main algorithmic challenges:\
(i) slow spread of information, i.e.\ many rounds (hops) of communication are required to spread information to all nodes in the network~\cite{pmlr-v97-assran19a,Vogels2021:relay}, and
(ii) heterogeneous data sources, i.e.\ when local data on each node is drawn from different distributions~\cite{karimireddy2019scaffold,hsieh2020non,Wang2021:fieldguide,bellet2021dcliques}.
The first challenge can partially be addressed by designing better mixing matrices (averaging weights) for a given network topology (i.e.\ selecting averaging coefficients and selecting subsets of neighbors for the information exchange). The general goal is to design a mixing matrix with a small spectral gap to ensure good mixing properties~\cite{Xiao2004:mixing,Duchi2012:distributeddualaveragig,Nedic2018:toplogy,pmlr-v97-assran19a,pmlr-v108-neglia20a}.
For addressing the latter challenge, the most widespread approach is to design specialized algorithms that can cope with data-heterogenity~\cite{Lorenzo2016GT-first-paper,Nedic2016DIGing,Tang2018:d2,Lin2021:quasiglobal}.


In contrast to these approaches, in this work, we investigate how the performance of D-SGD can be improved by a \emph{time-varying} and \emph{data-aware} design of the communication network (while respecting the network topology).
We propose a method that
adapts the mixing matrix to minimize \emph{relative heterogeneity}---the  gradient drift after averaging---when training on heterogeneous data. This allows to converge significantly faster and to reach a higher accuracy than state-of-the-art methods.
Our evaluations show that the additional benefit provided by the dynamic and data-adaptive design of the mixing matrix consistently outweighs other baselines on deep learning benchmarks, unlike other approaches based on drift-correction, and gradient tracking. 

We derive our approach guided  through theoretical principles. First, we refine the theoretical analysis of D-SGD to reveal precisely the tight interplay between the graph's mixing matrix and the time-varying distribution of gradients across nodes. We are not aware of any previous theoretical results that explore this connection. In the literature, the data heterogeneity and the communication graph have always been considered as separate parameters.

Our theoretical analysis shows that the design of an optimal data-dependent mixing matrix can be described as a quadratic program that can efficiently be solved. To make our approach practical and applicable to deep learning tasks, we propose a communication-efficient implementation via sketching techniques and intermittent communication.



Our main contributions are as follows:
\begin{enumerate}[leftmargin=12pt,nosep]
    \item We provide a tighter convergence analysis of DSGD by introducing a new metric that captures the interplay between the communication topology and data heterogeneity in decentralized (and federated) learning.
    \item We propose a communication and computation efficient algorithm to design data-aware mixing matrices in practice. 
    \item In a set of extensive experiments on synthetic and real data (ResNet20 \citep{he2015deep} on CIFAR10, Resnet18  \citep{he2015deep} on Imagenet, and fine-tuning distilBERT \citep{wolf2020huggingfaces} on AGNews), we show that our approach applies to modern large-scale Deep Neural Network training in decentralized settings leading to improved performance across various topologies. 
\end{enumerate}


\section{Related Work}



Decentralized convex optimization over arbitrary network topologies has been studied in \cite{Tsitsiklis1985:gossip,Nedic2009:distributedsubgrad,Wei2012:distributedadmm,Duchi2012:distributeddualaveragig}
and decentralized versions of the stochastic gradient method (D-SGD) have been analyzed in \cite{NIPS2017_lian, Wang2018:cooperativeSGD, Li2019:decentralized, koloskova20:unified}.
It was found that the convergence of D-SGD is strongly affected by heterogeneous data.
Such impacts are not only observed in practice~\cite{hsieh2020non,NIPS2017_lian}, but also verified theoretically by theoretical complexity lower bounds~\citep{WoodworthPS20,koloskova20:unified,karimireddy2019scaffold}.

Several recent works have attempted to tackle the undesirable effects of data heterogeneity across nodes on the convergence of D-SGD through suitable modifications to the algorithm.
D$^2$/Exact-diffusion~\cite{Tang2018:d2,yuan2020influence,yuan2021removing} apply variance reduction on each node. Gradient Tracking \citep{GT,Lorenzo2016GT-first-paper,pu2020distributed,lu2019gnsd,Koloskova2021:GT} utilizes an estimate of the full gradient at each node, obtained by successive mixing of gradients along with corrections based on updates to the local gradients. However, these approaches have not been found to yield performances comparable to D-SGD in practice \citep{Lin2021:quasiglobal}, despite superior theoretical properties \cite{Koloskova2021:GT, alghunaim2021unified}.
For optimizing convex functions, specialized variants such as EXTRA \cite{shi2015extra}, decentralized primal-dual methods \cite{Alghunaim2019:pd} have been developed.
With a focus on deep learning applications, \citet{Lin2021:quasiglobal,yuan2021decentlam} propose adaptations of momentum methods.






The undesirable effects of data heterogeneity persist also in the Federated Learning setting, which is a special case of the fully decentralized setting.
Several algorithms have been designed to tackle mitigate the undesirable effects of data heterogeneity \citep{karimireddy2019scaffold,wang2020tackling,mitra2021linear, dandi2021implicit}, yet extending them to the setup of decentralized learning remains challenging.

\citet{bellet2021dcliques} recently proposed utilizing a topology that minimizes the  data-heterogeneity across cliques composed of clusters of nodes capturing the entire diversity of data distribution (D-Clique). This allows having sparse connections across cliques while utilizing the full connectivity within each clique to ensure unbiased updates.
However, their approach assumes a fully connected underlying network topology (all nodes can reach each other in one hop), in contrast to the constrained setting we consider here. 
Our analysis does apply to their setting and can be used to theoretically explain the theoretical underpinnings behind D-Clique averaging.


Another line of work focuses on the design of (data-independent)  mixing matrices with good spectral properties~\cite{Xiao2004:mixing}. Another example is time-varying topologies such as the directed exponential graph \citep{pmlr-v97-assran19a} that allow for perfect mixing after multiple steps, or matchings~\cite{wang2019matcha}.
Several theoretical works argue to perform multiple averaging steps between updates~\cite{Scaman2017:optimal,lu21optimal,Kong2021:consensus}, though this introduces a noticeable overhead in practical DL applications. \citet{Vogels2021:relay} propose to replace gossip averaging with a new mechanism to spread information on embedded spanning trees.

Concurrent to our work, \citet{topologymatters} provided a similar analysis of the convergence rate of D-SGD for the smooth convex case using a metric quantifying the mixing error in gradients named ``neighborhood heterogeneity". 
Unlike our work, they focus on obtaining a fixed underlying sparse graph for the setting of classification under label skew using Frank-Wolfe \citep{fw1956}. Instead, our sketching-based approach allows efficient optimization of the mixing weights for a given graph for arbitrary heterogeneous settings during training in a dynamic manner. Our approach could also be utilized to learn sparse data-dependent topologies dynamically during training, such as with Frank-Wolfe methods \citep{fw1956,pmlr-v28-jaggi13} analogous to \citep{topologymatters}.




\section{Setup}
We consider optimizing the sum structured minimization objective distributed over $n$ nodes or workers/clients:
\begin{align}
  \min_{\xx\in \R^d} \sbr[\bigg]{ f(\xx)
:= \frac{1}{n}\sum_{i=1}^n f_i(\xx)}\,, 
\end{align}
where the functions $f_i(\xx) = \E_{\xi\sim \cD_i} F_i(\xx, \xi_i)$ denote the stochastic objectives locally stored on every node $i$. In machine learning applications, this corresponds to minimizing an empirical loss $f$ averaged over all local losses $f_i$, with~$\cD_i$ being a distribution over the local dataset on node $i$. 
We define a communication graph $\cG = (V, E)$ with $|V| = n$ are the nodes, and edges of this graph denote the possibility of communication, i.e. $(i,j) \in E$ only if nodes $i$ and $j$ are able to communicate.

Following a convention in decentralized literature~\citep[e.g.][]{Xiao2004:mixing}, we define a mixing matrix $W \in \R^{n \times n}$ as a weighted adjacency matrix of $\cG$ with the weights $w_{ij} \in [0, 1]$, $w_{ij} > 0$ iff $(i,j) \in E$ and the matrix is doubly stochastic $\sum_{i = 1}^n w_{ij} = 1$.

 In D-SGD, every worker $i \in [n]$ maintains local parameters $\xx_i^{(t)} \in \R^d$ that are updated in each iteration with a stochastic gradient update (computed on the local function $f_i$) and by averaging with neighbors in the communication graph. It is convenient to compactly write the gradients in matrix notation:
\begin{align}
    X^{(t)} &:= \left[ \xx_1^{(t)},\dots, \xx_n^{(t)}\right]  \in \R^{d \times n}\,, \notag \\
\partial F(X^{(t)}, \xi^{(t)}) &:= \left[\nabla F_1(\xx_{1}^{(t)}, \xi_1^{(t)}), \dots,  \nabla F_n(\xx_{n}^{(t)}, \xi_n^{(t)})\right]\,, \notag \\
\partial f(X) &:= \left[\nabla f_1(\xx_{1}), \dots,  \nabla f_n(\xx_{n})\right]
\label{eq:matrix_notation}
\end{align}%
where $\xi^{(t)}$ are independent random variables such that $\Ea{\partial F(X^{(t)}, \xi^{(t)})} = \partial f(X)$. Similarly, we denote the mixing step as multiplication with the mixing matrix $W$. This is illustrated in Algorithm~\ref{alg:dsgd}.

\begin{algorithm}[ht]
\let\oldendfor\algorithmicendfor
\renewcommand{\algorithmicendfor}{\algorithmicend\ \textbf{parallel for}}
\let\olddo\algorithmicdo
\renewcommand{\algorithmicdo}{\textbf{do in parallel on all workers}}
    \begin{minipage}{\linewidth}
	\caption{\textsc{Decentralized SGD }}\label{alg:dsgd}
	\begin{algorithmic}[1]
		\INPUT{$X^{(0)}$, stepsizes $\{\eta_t\}_{t=0}^{T-1}$, number of iterations $T$, mixing matrix distributions $\cW^{(t)}$, $t \in \{0, \dots, T\}$} 
		\FOR{$t$\textbf{ in} $0\dots T$}
		\STATE $G^{(t)} = \partial F(X^{(t)}, \xi^{(t)})$ \hfill $\triangleright$ stochastic gradients
		\STATE $W^{(t)} \sim \cW^{(t)}$ \hfill  $\triangleright$ sample mixing matrix
		\STATE $X^{(t + 1)} = (X^{(t)} -\eta_t G^{(t)}) W^{(t)}$ \hfill  $\triangleright$ update \& mixing
		\ENDFOR
	\end{algorithmic}
	\end{minipage}
\let\algorithmicendfor\oldendfor
\let\algorithmicdo\olddo
\end{algorithm}


On line 2 of  Algorithm~\ref{alg:dsgd} every node in parallel calculates stochastic gradients, on line 3 a mixing matrix is sampled from the distribution $\cW$, and on line 4, every node performs a local SGD update and after that mixes updated parameters with the sampled matrix $W^{(t)}$. 

\subsection{Standard Assumptions}
We use the following assumptions on objective functions:
\begin{assumption}[$L$-smoothness]\label{a:lsmooth}
Each local function $F_i(\xx, \xi)\colon \R^d \times \Omega_i \to \R$, $i \in [n]$
is differentiable for each $\xi \in \supp(\cD_i)$ and there exists a constant $L \geq 0$ such that for each $\xx, \yy \in \R^d, \xi \in \supp(\cD_i)$:
\begin{align}
&\norm{\nabla F_i(\yy, \xi) - \nabla F_i(\xx,\xi) } \leq L \norm{\xx -\yy}\,. \label{eq:F-smooth}%
\end{align}
\end{assumption}

Sometimes we will assume ($\mu$-strong) convexity on the functions $f_i$ defined as
\begin{assumption}[$\mu$-convexity]
\label{a:strong}
Each function $f_i \colon \R^d \to \R$, $i \in [n]$ is $\mu$-(strongly) convex for constant $\mu \geq 0$. That is, for all $\xx,\yy \in \R^d$:
\begin{align}
 f_i(\xx)-f_i(\yy) + \frac{\mu}{2}\norm{\xx-\yy}^2_2 \leq \lin{\nabla f_i(\xx),\xx-\yy}\,. \label{eq:strongconv}
\end{align}
\end{assumption}

\begin{assumption}[Bounded Variance]\label{a:stoc}
We assume that there exists a constant   $\sigma $ such that $\forall \xx \in \R^d$
	\begin{align} \textstyle
	\frac{1}{n} \sum_{i = 1}^n \EE{\xi_i \sim \cD_i}{\norm{\nabla F_i(\xx,\xi_i) - \nabla f_i(\xx)}}^2_2 \leq  \sigma^2\,. \label{eq:noise_opt_nc}
	\end{align}
For the convex case it suffices to assume a bound on the stochasticity at the optimum $\xx^\star := \argmin f(\xx)$. We assume there exists a constant $\sigma_\star^2 \leq \sigma^2$, such that 
	\begin{align} \textstyle
	\frac{1}{n} \sum_{i = 1}^n \EE{\xi_i \sim \cD_i}{\norm{\nabla F_i(\xx_\star,\xi_i) - \nabla f_i(\xx_\star)}}^2_2 \leq \sigma_\star^2 \label{eq:noise_opt}
	\end{align}
\end{assumption}

\begin{assumption}[Consensus Factor]
\label{a:p}
We assume that there exists a constant $p \in (0,1]$ such that for all $t \geq 0$:
\begin{align}
    \E_{W\sim \cW^{(t)}} \big\| X W - \overline X \big\| _F^2 \leq (1-p) \big\| X - \overline X \big\|_F^2 \,,
\end{align}
for all $X \in \R^{d \times n}$ and $\overline X:= X\frac{\1\1^\top}{n}$.
\end{assumption}

\begin{remark}[Spectral Gap]
The spectral gap of a fixed mixing matrix $W$ is defined as
$\delta = 1-\norm{W}_2$ and it holds, $\| X W - \overline X \big\|_F^2 \leq (1-\delta) \| X - \overline X \|_F^2$. That is, for a fixed~$W$, the spectral gap gives a lower bound on the consensus factor.
\end{remark}

\subsection{Gradient Mixing}
Prior work introduced various notions to measure the \emph{dissimilarity} between the local objective functions~\cite{NIPS2017_lian, Tang2018:d2, lu21optimal, koloskova20:unified}. 
For instance, \citet{NIPS2017_lian, Tang2018:d2, koloskova20:unified} introduce the heterogeneity parameter $\zeta^2$ (and for the convex case $\zeta_\star^2 \leq \zeta^2$) satisfying\footnote{\citet{koloskova20:unified} use a slightly more general notion which we omit here for conciseness. We show in the appendix that our results extend to their notion as well.}
\begin{align}
	& \tfrac{1}{n} \big\| \partial f( \overline X)  - \overline{\partial f}( \overline X) \big\|_F^2  \leq \zeta^2 \qquad \forall X \in \R^{d\times n}\,,
	\label{eq:hat-zeta} \\
	& \tfrac{1}{n} \big\| \partial f( X_\star) - \overline{\partial f}( X_\star) \big\|_F^2  \leq  \zeta^2_\star\,.
	\label{eq:star-zeta},
\end{align}
and $\overline X = X \frac{\1 \1^\top}{n}$, 
$\overline{\partial f}(X) = \partial f(X) \frac{\1 \1^\top}{n}$, $X_\star= \xx_\star \1^\top$. Here $\1 \in \R^d$ denotes the all-one vector.
The measures in~(\ref{eq:hat-zeta})--(\ref{eq:star-zeta}) do not depend on the mixing matrix.
We instead propose to measure heterogeneity relative to the graph connectivity and the choice of the mixing matrix.
\begin{assumption}[Relative Heterogenity]
\label{a:relative}
We assume that there exists a constant $\zeta'$, such that $\forall X \in \R^{d\times n}$, $\forall t \geq 0$:
\begin{align}
	& \E_{W \sim \cW^{(t)}} \tfrac{1}{n}\big\| \partial f( \overline X) W - \overline{\partial f}( \overline X) \big\|^2 \leq \zeta'^2\,.
	\label{eq:hat-zeta-prime} 
\end{align}
For the convex case, it suffices to assume a bound at the optimum $X_\star $ only. We assume there exists a constant $\zeta_\star'^2 \leq \zeta'^2$, such that $\forall t \geq 0$\vspace{-1mm}
\begin{align}
	&\E_{W \sim \cW^{(t)}} \tfrac{1}{n}\big\| \partial f( X_\star) W - \overline{\partial f}( X_\star) \big\|^2 \leq  \zeta_\star'^2\,.
	\label{eq:star-zeta-prime}
\end{align}
\end{assumption}
The above quantity measures the effectiveness of a mixing matrix in producing close to the global average of the gradients at each node.
We now show that the new relative heterogeneity measure is always lower than the heterogeneity parameters used in prior work. 
\begin{remark} Using Assumption 4, we obtain:
\begin{align*}
    &\E_{W \sim \cW^{(t)}}\tfrac{1}{n}\big\| \partial f( \overline X) W - \overline{\partial f}( \overline X) \big\|^2 \\ & \qquad= \E_{W \sim \cW^{(t)}}\tfrac{1}{n}\big\| ( \partial f( \overline X) - \overline{\partial f}( \overline X) ) (W-\tfrac{1}{n}\1\1^\top) \big\|^2  \\
    &\qquad\leq \tfrac{1}{n} (1-p)\big\| \partial f( \overline X) - \overline{\partial f}( \overline X) \big\|^2 \leq (1-p)\zeta^2\,.
\end{align*}
This implies that we can always choose ${\zeta}'^2\leq (1-p)\zeta^2$ and $\zeta_\star'^2 \leq (1-p)\zeta_\star^2$. Often $\zeta'$ can even be much smaller (see the discussion in Section~\ref{sec:zetadiscussion} below).
\end{remark}



\section{Convergence Result}
In this section, we present a refined analysis of the D-SGD algorithm \cite{NIPS2017_lian,koloskova20:unified}.
We state our main convergence results below, whose proofs can be found in the Appendix. These results are stated for the average of the iterates, $\overline \xx^{(t)}:=\frac{1}{n}\sum_{i=1}^n \xx_i^{(t)}$.
\begin{theorem}\label{thm:rate} 
Let Assumptions~\ref{a:lsmooth}, \ref{a:stoc}, \ref{a:p} and~\ref{a:relative} hold. Then there exists a stepsize $\eta \leq \frac{p}{L}$ such that Algorithm~\ref{alg:dsgd} needs the following number of iterations to achieve an $\epsilon$ error:\\
\textbf{Non-Convex:} It holds $\frac{1}{T + 1}\sum_{t = 0}^T \E \big\|\nabla f(\overline{\xx}^{(t)})\big\|_2^2 \leq \epsilon$ after
\resizebox{\linewidth}{!}{
\vbox{
\begin{align*}
 \cO \bigg( \frac{ \sigma^2}{n \epsilon^2}  +  \frac{ \zeta' + \sigma \sqrt{p}   }{ p \epsilon^{3/2}}   + \frac{1}{p \epsilon}  \bigg) \cdot L F_0
\end{align*}
}}
iterations. If we in addition assume convexity,\\
\textbf{Convex:} Under Assumption~\ref{a:strong} for $\mu \geq 0$, the error $ \frac{1}{(T+1)} \sum_{t=0}^T  (\E{f(\overline\xx^{(t)})} - f^\star) \leq \epsilon$ after
\begin{align*}
   \cO \bigg( \frac{\sigma^2 }{n \epsilon^2} + \frac{\sqrt{L}( \zeta' + \sigma \sqrt{p} )  }{ p \epsilon^{3/2}}   + \frac{L}{p \epsilon} \bigg) \cdot R_0^2
\end{align*}
iterations, and  if $\mu > 0$, \\
\textbf{Strongly-Convex:}
then $ \sum_{t= 0}^T \frac{w_t}{\sum_{t=0}^T w_t} (\E{f(\overline\xx^{(t)})} - f^\star) + \mu \E {\| \overline{\xx}^{(T+1)} - \xx^\star} \|^2 \leq \epsilon$ for\footnote{$\tilde{\cO}/\tilde{\Omega}$-notation hides constants and polylogarithmic factors.}
 \begin{align*}
  \tilde \cO \bigg( \frac{\sigma^2}{ \mu  n \epsilon} + \frac{\sqrt{L}( \zeta'_\star + \sigma_\star \sqrt{p} )  }{\mu p  \sqrt{\epsilon}}   + \frac{L}{\mu p} \log \frac{1}{\epsilon} \bigg)
 \end{align*}
 iterations, where $w_t$ denote appropriately chosen positive weights,
 $F_0 := f(\xx_0)- f^\star$ for 
 $f^\star = \min_{\xx \in \R^d}f(\xx)$ and $R_0=\norm{\xx_0 -\xx^\star}$ denote the initial errors. 


\end{theorem}
We prove this theorem in the appendix. 

\subsection{Discussion}
\label{sec:zetadiscussion}
First, we note that our convergence rates in Theorem~\ref{thm:rate} look similar to the ones in \cite{koloskova20:unified} but the old heterogeneity $\zeta$ (or $\zeta_\star$) is replaced with the new relative heterogeneity measure $\zeta'$ (or $\zeta_\star'$ correspondingly). As $\zeta' \leq \zeta$ ($\zeta_\star' \leq \zeta_\star$), the convergence rate given in Theorem~\ref{thm:rate} is always tighter than previous works.  

We now argue that $\zeta'$ can be significantly smaller than $\zeta$.

As a motivating example, we consider a ring topology with the  Metropolis-Hasting mixing weights and a particular pattern on how the data is distributed across the nodes:

\begin{example}
\label{ex:1}
Consider a ring topology on $n=3k$ nodes, $k\geq 1$, with uniform mixing among neighbors ($w_{i,i-1}=w_{i,i}=w_{i,i+1}=\frac{1}{3}$) and assume that $\cD_i = \cD_{i+3 \mod n }$ for all $i$ and suppose there is an $\xx'$ with $\nabla f(\xx')=0$, $\| \nabla f_1(\xx')\|>0$. Then $\zeta'=0$ and $\zeta  \neq 0$.
\end{example}
This is easy to see that the relative heterogenity is $\zeta'=0$. This holds, because uniform averaging of three neighboring gradients result in an unbiased gradient estimator:
\begin{align*}
    \textstyle\tfrac{1}{3}\sum_{{j \in \{i-1,i,i+1\} }} \nabla f_{j}(\xx)   = \nabla f(\xx)\,, \qquad\forall \xx \in \R^d\,,
\end{align*}
 while in contrast
\begin{align*}
    \zeta^2 \geq \tfrac{1}{3}\left(\norm{\nabla f_1(\xx')}^2 + \norm{\nabla f_2(\xx')}^2 + \norm{\nabla f_3(\xx')}^2\right) > 0.
\end{align*}
Although this example is somewhat artificial, it is not hard to imagine something similar happening in practice, e.g.\ if the nodes are randomly assigned to some graph topology. 

A second example is motivated by~\citet{bellet2021dcliques}, who propose an algorithm that constructs a sparse topology depending on the data, in which nodes are organized in interconnected cliques (i.e., locally fully connected sets of nodes) such that the joint label distribution of each clique is close to that of the global distribution.
\begin{example}[Perfect clique-averaging]
\label{example:clique}
Suppose the graph topology can be divided into $k\geq 1$ cliques $C_1, \dots, C_k$ such that for every clique it holds $\sum_{j \in C_i} \nabla f_j(\xx)=\nabla f(\xx)$. 
Then by designing the mixing matrix such that it corresponds to uniform averaging within each clique results in  $\zeta'=0$.
\end{example}

\subsection{Designing good mixing matrices}\label{sec:design}
One of the main advantages of our theoretical analysis is that it allows a principled design of good mixing matrices. 
We identify in Theorem~\ref{thm:rate}  two concurrent factors:\ on the one hand, the consensus factor $p$ should be close to $1$, and on the other hand the relative heterogeneity parameter $\zeta'$ should be close to $0$. Trying to find a mixing matrix satisfying both might seem a difficult task. However, one can combine matrices that are good for either of the tasks.

\begin{example}
\label{example:product}
Suppose a mixing matrix $W_p$ has consensus factor $p \leq 1$, and a mixing matrix $W_{\zeta'}$ has relative heterogeneity parameter $\zeta'$. Then $W=W_{\zeta'} W_p$ has consensus factor at least $p$ and relative heterogeneity at most $\zeta'$.
\end{example}
\begin{proof}
By the mixing property of $W_p$,\vspace{-1mm}
\begin{multline*}
\norm{XW - \overline X}_F^2  = \norm{X W_{\zeta'} W_p - \overline X}_F^2 \\   \leq (1 - p) \norm{X W_{\zeta'} - \overline X}_F^2  
\leq (1 - p)\norm{X - \overline X}_F^2\,,
\end{multline*}
and similarly, \vspace{-1mm}
\begin{align*}
 &\tfrac{1}{n}\big\| \partial f( \overline X) W_{\zeta'} W_p - \overline{\partial f}( \overline X) \big\|^2 \\
 &\qquad \leq \tfrac{1}{n}\big\| \partial f( \overline X) W_{\zeta'} - \overline{\partial f}( \overline X) \big\|^2 \leq \zeta'^2 \,.  \qedhere
\end{align*}
\end{proof}



Where we used the double stochasticity of $W_{\zeta'}$ which implies that $\norm{(W-\tfrac{1}{n}\1\1^\top)}_2 \leq 1$ (Proof in Proposition \ref{prop:perron} in the appendix).
In practice, we observe that two communication rounds are not necessary, alternating between mixing with $W_p$ and $W_{\zeta'}$ works well and does not increase the communication costs.

\subsection{Possible Theory Extensions}
Theorem~\ref{thm:rate} does not cover the just discussed case of alternating between two or more matrices. As our main focus in this work is on highlighting the benefits of relative heterogeneity, we just covered a simple case of time-varying mixing in the theorem (when all matrices are sampled from the same distribution). However, it is possible to extend our analysis to deterministic sequences (such as alternating) with the derandomization technique presented in~\citep[Assumption 4, Theorem 2]{koloskova20:unified}.

Another case that is not covered, but possible to cover, is the use of two separate mixing matrices to mix parameters and gradients respectively (similar as in \citealt{bellet2021dcliques}).
However, this scheme requires two rounds of communications (similar to Example~\ref{example:product} that is covered in our analysis).

\section{Heterogeneity-Aware Mixing}

The results presented above were based upon a fine-grained analysis of Algorithm~\ref{alg:dsgd} that incorporates the effects of relative heterogeneity. In this section, we build upon our novel theoretical insights developed above to improve the performance of D-SGD in practice.

\subsection{Motivation}

Theorem~\ref{thm:rate} predicts that small values of the relative heterogeneity parameter $\zeta'$ lead to improved convergence. More specifically, progress in each iteration is determined by the current data-dependent \emph{gradient mixing error}
\[
\big\| \partial f(\overline{X}^{(t)})W^{(t)}- \overline{\partial f}(\overline{X}^{(t)}) \big\|^2 \,,
\]
which is upper bounded by $\zeta'$ (as defined in Assumption~\ref{a:relative}). This quantity depends both on the current iterate $X^{(t)}$ but also on the chosen mixing weights $W^{(t)}$, thus suggesting to continually update the mixing matrix such that the gradient mixing error remains low, while the gradients evolve during training.

Thus, we can write the following time-varying optimization problem for the mixing weights $W$. For current parameters $X \in \R^{d \times n}$, $\overline X = X \frac{1}{n} \1\1^\top$ (we drop the time index) distributed over $n$ nodes, we aim to solve 
\begin{align}\label{eq:gme}
\begin{split}
    & \min_{W \in \cM_w} \norm{\partial f(\overline{X})W- \overline{\partial f}(\overline{X})}_F^2
\end{split}\tag{GME-exact}
\end{align}
where $\cM_w = \{W: \1 W \!=\! \1 \ , \1^\top W \!=\! \1^\top; 0 \leq  w_{ij} \leq 1 ~~ \forall i,j\,, ~~ w_{ij} \!=\! 0 ~~\forall (i,j) \not\in E  \}$ is the set of allowed mixing matrices. The objective function comes from the definition of ${\zeta}'^2$ in Equation~(\ref{eq:hat-zeta-prime}). The first two conditions ensure double stochasticity of $W$, while the last condition respects edge constraints of the communication graph $\cG$. Note that unlike the matrix corresponding to the optimal spectral gap, the optimal matrix obtained above could be asymmetric. We call this optimization problem the exact Gradient Mixing Error~(\ref{eq:gme}). 

\subsection{Proposed Algorithm}
We can equivalently reformulate (\ref{eq:gme}) as to more efficiently solve when the dimension $d$ of the gradient vectors is large, compared to the number of nodes $n$: 
\begin{align}\label{eq:gme-opt}
   \min_{W \in \cM_w} \Tr\left[W^\top \Gamma W \right]
\tag{GME-opt-$\Gamma$}
\end{align}
where 
\begin{align*}
  \Gamma := \rbr*{\partial f(\overline{X})- \overline{\partial f}(\overline{X})}^\top\rbr*{\partial f(\overline{X})- \overline{\partial f}(\overline{X})} \,.
\end{align*}
This is a quadratic program with linear constraints. The minimizer, i.e.\ resulting mixing matrix, of (\ref{eq:gme-opt}) is the same as  for~(\ref{eq:gme}). However, as the problem formulation depends only on the gram matrix $\Gamma \in \R^{n \times n}$ it can be solved more efficiently \cite{boyd2004convex}. 

However, directly attempting to solve
(\ref{eq:gme-opt}) would not result in an efficient algorithm in our  decentralized setting, for serveral reasons: (i) the nodes do not have access to the parameters average $\overline \xx$;  (ii) no access to the expected (full-batch) gradients (iii) too costly to transmit all gradients into one place to compute the inner products; (iv) too costly to solve this problem at every iteration. 

To address these issues, we propose to \emph{approximately} solve a more efficient sketch of the following objective, only once every $H \geq 1$ step. We summarise the resulting algorithm in Algorithm~\ref{alg:dsgd-gme}, which calls \eqref{eq:gme-opt} as a subproblem.

\subsection{Justifying the Design Choices}
Next we analyze the relationship between \ref{eq:gme-opt} and \ref{eq:gme}.

\begin{algorithm}
\let\oldendfor\algorithmicendfor
\renewcommand{\algorithmicendfor}{\algorithmicend\ \textbf{parallel for}}
\let\olddo\algorithmicdo
\renewcommand{\algorithmicdo}{\textbf{do in parallel on all workers}}
    \begin{minipage}{\linewidth}
	\caption{\textsc{Heterogeneity-aware Decentralized SGD} (HA-SGD)}\label{alg:dsgd-gme}
	\begin{algorithmic}[1]
		\INPUT{$X^{(0)}$, stepsizes $\{\eta_t\}_{t=0}^{T-1}$, number of iterations $T$, communication graph $G$} 
		\FOR{$t$\textbf{ in} $0\dots T$}
		\STATE $G^{(t)} = \partial F(X^{(t)}, \xi^{(t)})$ \hfill $\triangleright$ stochastic gradients
		\IF{$t \mod H = 0$}
		\STATE $W^{(t)} = \operatorname{CE-GME}(G^{(t)})$
		\ELSE
		\STATE $W^{(t)} = W^{(t-1)}$
		\ENDIF
		\STATE $X^{(t + 1)} = (X^{(t)} -\eta_t G^{(t)}) W^{(t)}$ \hfill  $\triangleright$ update \& mixing
		\ENDFOR
	\end{algorithmic}
	\end{minipage}
\let\algorithmicendfor\oldendfor
\let\algorithmicdo\olddo
\end{algorithm}


\begin{algorithm}
\let\oldfor\algorithmicfor
\renewcommand{\algorithmicfor}{\textbf{in parallel} }
\let\oldendfor\algorithmicendfor
\renewcommand{\algorithmicendfor}{\algorithmicend}
    \begin{minipage}{\linewidth}
	\caption{CE-GME: Communication Efficient GME }\label{alg:gme}
	\begin{algorithmic}[1]
		\INPUT{matrix $G \in \R^{d \times n}$, distributed column-wise across $n$ nodes, random seed $s$, dimension $k$} 
		\FOR{on $n$ nodes}
		\STATE sample $A \in \R^{k \times d}$, $a_{ij} \sim \cN(0, 1)$ \hfill $\triangleright$ use the same random seed $s$ on every node. 
		\STATE $S=A G \in \R^{k \times n}$ \hfill $\triangleright$ compute sketches
		\STATE $\overline{S}=S \frac{\1\1^\top}{n}$ \hfill $\triangleright$ all-reduce-communication
 		\STATE $\Gamma = \rbr*{S-\overline{S}}^\top\rbr*{S-\overline{S}}$ \hfill $\triangleright$ sketched gram matrix
		\STATE $W=\operatorname{GME-opt}(\Gamma)$ 
		\ENDFOR
	\end{algorithmic}
	\end{minipage}
\let\algorithmicfor\oldfor
\let\algorithmicendfor\oldendfor
\end{algorithm}

\paragraph{Effect of Periodic Optimization.}
As a first distinction from \ref{eq:gme}, we propose to optimize the mixing matrix $W$ only once every $H$ steps in order to reduce the computational cost. Below we show that for small $H$, if at step $t + H$ we apply the matrix $W^{(t)}$ found by GME at the step~$t$, then this matrix would still give a good error~$\zeta'$.

To distinguish only the effect of periodic optimization, we assume that every $H$ steps we solve an original \ref{eq:gme} problem. We perform optimization using the full gradients, moreover on line 4 of Algorithm~\ref{alg:dsgd-gme} we solve an original (GME-exact) problem with full gradients on the averaged parameters, i.e. line 4 is replaced with $W^{(t)} = \operatorname{GME}(\partial f(\overline X^{(t)}))$. 

\begin{proposition}\label{lem:period_effect}
$ \big\|\partial f\rbr*{\overline{X}^{(t+H)}} W^{(t)} -\overline{\partial f}\rbr*{\overline{X}^{(t+H)}}\big\|^2_F \\
 \qquad  ~\leq~ 2 \big\| \partial f\rbr*{\overline{X}^{(t)}} W^{(t)} -\overline{\partial f}\rbr*{\overline{X}^{(t)}} \big\|^2_F \\
 ~~~~~~~+ 2H\sum_{i=0}^{H-1}\eta_t^2L^2\norm{\partial f(X^{(t+i)})}^2_F $ 
\end{proposition}
 
For the proof refer to appendix. Since the learning rate $\eta_t$ is usually small, the relative heterogeniety does not increase much for a small number of steps~$H$. 

\paragraph{Effect of Stochastic Estimation.}
In practice the full gradients are too expensive to compute, so we will resort to stochastic gradients instead. The following proposition controls the error due to the selection of the mixing matrix using stochastic gradients.
\begin{proposition}\label{prop:stoch}
Let $W^{*}(\xi)$ be any mixing matrix satisfying given edge constraints dependent on the noise parameters~$\xi$. Then, we have:
\begin{align*}
&\Ea{\norm{\rbr*{\partial f(\overline{X})- \overline{\partial f}(\overline{X})}W^{*}(\xi)}_F^2}
\\&\leq 2\Ea{\norm{\rbr*{\partial f(\overline{X},\xi)- \overline{\partial f}(\overline{X},\xi)}W^{*}(\xi)}_F^2} + 2n\sigma^2.
\end{align*}
\end{proposition}
Proof can be found in the appendix. Setting $W^{*}(\xi) = {\displaystyle\argmin_{W \in \cM_w}}\norm{\partial f(\overline{X},\xi) W - \overline{\partial f}(\overline{X},\xi)}_F^2$ reveals that minimizing GME with stochastic gradients would also lead to a small heterogeneity $\zeta$ up to additive stochastic noise.

\begin{figure*}[ht!]
    \centering
     \begin{subfigure}[b]{0.65\columnwidth}
    \includegraphics[width=\textwidth]{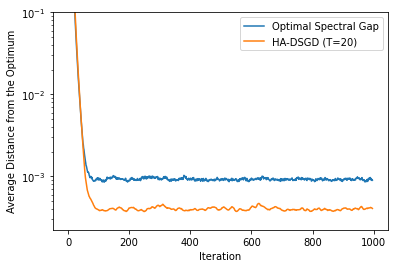}
    \end{subfigure}
    \begin{subfigure}[b]{0.65\columnwidth}
    \includegraphics[width=\textwidth]{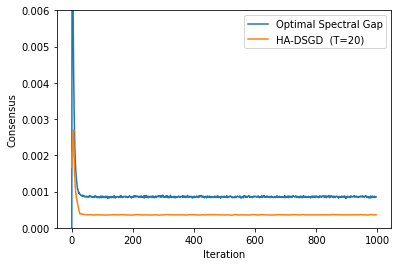}
    \end{subfigure}
    \begin{subfigure}[b]{0.65\columnwidth}
    \includegraphics[width=\textwidth]{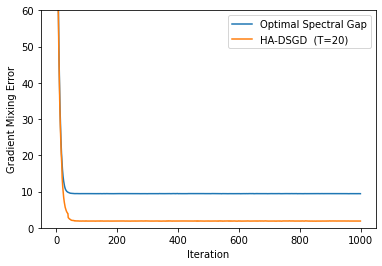}
    \end{subfigure}
   \vspace{-2mm}
    \caption{Comparison of HA-DSGD to D-SGD. (a) Average distance from the optimum, (b) consensus distance $\smash{\frac{1}{n}\big\|X - \overline X\big\|_F^2}$, and (c) gradient mixing error $\smash{\big\|\partial F({X}, \xi) W- \overline{\partial F}({X}, \xi)\big\|_F^2}$ vs.\ the number of iterations for quadratic objectives. ``Optimal Spectral Gap'' denotes 
    the DSGD algorithm with mixing matrix optimized for a spectral gap. 
    We report an average over a window of~5 iterations of corresponding quantity on each plot. 
    }
    \label{fig:quadratics}
\end{figure*}

\paragraph{Sketching for Gram Matrix Estimation.}

The original GME-exact formulation requires transmitting the entire gradients. We instead propose to calculate the Gram matrix using sketched gradients, for improved communication efficiency. 

Let $A$ denote a random matrix with Gaussian entries and let $U$ be an arbitrary matrix. We observe that $\E (UA)^\top UA = \E U^\top A^\top A U = U^\top U$. Therefore, the above projection operation preserves the inner products in expectation.
The approximation error of the above scheme can be bounded using the following extension of the  Johnson–Lindenstrauss lemma, whose proof can be found in the appendix:

\begin{proposition}\label{lem:compress}
Let $\{\uu_1,\cdots,\uu_m\} \in \mathbb{R}^{d}$. Assume that the entries in $A \subset \mathbb{R}^{k \times d}$ are sampled independently from $\mathcal{N}(0,1)$. Then, for $k = \omega(\frac{log(\frac{m}{\delta})}{\epsilon^2})$, with probability greater than $1-\delta$ , we have:
\begin{align*}
\abs{\tfrac{1}{k} \lin{A\uu_i,A\uu_j}  - \lin{\uu_i,\uu_j}} \leq \epsilon \max_{i \in [m]} \norm{\uu_i}^2
\end{align*}
for all $i,j \in [m]$.
\end{proposition}

See appendix for the proof. In our algorithm, the $\{\uu_1,\cdots,\uu_m\} \in \mathbb{R}^{d}$ correspond to the gradients across nodes, and are compressed using a 
Gaussian matrix generated, independently at each period using shared seeds.


\paragraph{Use of local $X$.} 
In our practical implementation we solve GME problem for gradients computed at the parameters $X$ instead of $\overline X$ in \ref{eq:gme}. We show that this leads to the minimization of the GME upto an additional term proportional to the consensus:

\begin{proposition}\label{prop:local}
$
    \norm{\partial f(\overline X) W - \overline{\partial f}(\overline X)}_F^2 $\\$~~\leq~~ 2\norm{\partial f( X)W - \overline{\partial f}(X)}_F^2 
     + 2L^2 \norm{X - \overline X}_F^2
$
\end{proposition}

We prove this proposition in the appendix. We also give an estimate of  the decrease of consensus distance $\|X - \overline X\|_F^2$. Thus, the small right hand side ensures the small relative heterogeneity.

\subsection{Discussion \& Extensions}\label{sec:extend}

\textbf{Optimizing Mixing of updates of arbitrary algorithms}: Our approach can be generalized to arbitrary additive updates to the parameters of the form $\xx_i^{(t+1)}=\xx_i^{(t)}+\eta \uu_i^{(t)}$. For example, replacing the gradients in the Algorithm \ref{alg:dsgd-gme} by the updates of the Adam algorithm \citep{DBLP:journals/corr/KingmaB14} results in the minimization of the mixing error involved in decentralized Adam updates. We empirically verify the effectiveness of such an algorithm for an NLP task as discussed in Section~\ref{sec:nlp}. 

\textbf{Optimizing Mixing of Parameters}: An alternate way of simultaneously maximizing the consensus factor $p$ and the gradient mixing error is to directly optimize the mixing error of the parameters i.e. $\big\|(X^{(t)} - \overline X^{(t)})W\big\|_F^2$.  Our theoretical analysis covers such a choice of mixing matrices as a special case that involves trying to obtain a mixing matrix having both small $(1-p)$ and the gradient mixing error. 
However, unlike the gradient mixing error that involves changes of the order $\eta^2$ as shown by Lemma \ref{lem:period_effect},  the distribution of the parameters across nodes can change rapidly due to the mixing. Moreover, we found both approaches to yield similar improvements in practice and focus on the gradient mixing error since it covers a wider range of design choices such as mixing within unbiased cliques.
    
    

\begin{table*}[!htbp]
    \centering
    \begin{tabular}{l%
    >{\centering\arraybackslash}p{2.5cm}%
    >{\centering\arraybackslash}p{2.5cm}%
    >{\centering\arraybackslash}p{3.2cm}}
    \toprule
    Method & Ring (n=$16$) & Torus (n=$16$) & Social Network (n=$32$) \\
    \midrule
    DSGD & $74.71 \pm 2.24$ & $76.13 \pm 1.65$ &$77.68 \pm 1.42$\\
    HA-DSGD &$78.21 \pm 2.19$ & $79.08 \pm 2.07$  & $79.54 \pm 1.61$\\
    HA-DSGD (momemtum, period=100) & $80.75 \pm 1.84$ & $82.22 \pm 1.87$ & $83.24 \pm 1.15$ \\
    DSGD (momentum, Metropolis-Hastings)    & $77.52 \pm 2.78$ &  $80.45 \pm 2.27$ & $80.71 \pm 1.93$ \\
    DSGD (momentum, Optimal Spectral Gap) & $79.06 \pm 1.82$ &$80.28 \pm 2.12$ & $80.91 \pm 1.74$\\
    $D^2$ & $49.68 \pm 3.19$ & $51.37 \pm 2.68$  & $52.15 \pm 2.43$\\ \bottomrule
    \end{tabular}
    \caption{Top-1 test accuracy on CIFAR10 under different topologies. The results in the table are averaged
over three random seeds. } 
    \label{tab:cifar}
\end{table*}
    
\section{Experiments}
For all our experiments, we use the CVXPY \citep{diamond2016cvxpy} convex optimization library to perform the constrained optimization defined in the section. In all our results, the period denotes the number of updates after which the mixing matrix is recomputed i.e.\ a period of $100$ implies that the communication of the compressed gradients and the computation of the mixing matrix occurs only for a $\frac{1}{100}$ fraction of the updates. We denote the number of nodes in the underlying topology by $n$. HA-DSGD refers to our proposed Alg.~\ref{alg:dsgd-gme} with updates alternating between the weights obtained by the GME optimization and the Metropolis-Hastings weights, similarly as discussed in Section \ref{sec:design}.

\subsection{Quadratic Objectives}\label{sec:quad}
We first consider a simple setting of random quadratic objectives, with the objective for the $i_{th}$ client given by $
    f_i(\xx) = \norm{A_i\xx+b_i}^2_2
    $,
where $\xx$ denotes a $d$ dimensional parameter vector and both $A_i$ and $b_i$ contain entries sampled randomly from $\mathcal{N}(0,1)$ and fixed for each client. For our experiments, we set $d=10$. We further introduce stochasticity to the gradients by adding random Gaussian noise with variance $0.1$. We generate a random connected graph of $16$ nodes by randomly removing half of the edges from a complete graph, while ensuring that the connectivity is maintained. Figure~\ref{fig:quadratics} illustrates the improvements due to our approach across three metrics: the distance from the optimum, consensus error, as well as the gradient mixing error.

\subsection{Deep Learning Benchmarks}

We evaluate our approach on computer vision as well as natural language processing benchmarks. Following \citet{pmlr-v97-yurochkin19a}, for each setting, the heterogeneity across clients is governed by a Dirichlet distribution-based partitioning scheme with a parameter $\alpha$ quantifying the dissimilarity between the data distributions across nodes. We set $\alpha=0.1$ for all the experiments since it corresponds to a setting with high heterogeneity. We compare against the baseline DSGD with local momentum under mixing weights defined by the Metropolis-Hastings scheme as well as those obtained through the optimization of the spectral gap \citep{Boyd03fastestmixing}. We use standard learning rate scaling and warmup schedules as described in \citep{goyal2018accurate} for the computer vision tasks and use a constant learning rate with Adam optimizer \citep{DBLP:journals/corr/KingmaB14} for the NLP task. Further experimental details for all the settings are outlined in Appendix \ref{app:arch}.

\begin{table}[t]
    \centering
    \begin{tabular}{cc}
    \toprule
    Compression Dimension & Test Accuracy\\
    \midrule
        1 &  $75.66$\\
        100 & $81.55$\\
        1000 & $81.97$ 
    \\ \bottomrule
    \end{tabular}
    \caption{Effect of the Compression Dimension on the top-1 test accuracy on the CIFAR dataset.}
    \label{tab:compress}
\end{table}

\begin{table}[t!]
    \centering
    \begin{tabular}{lcc}
    \toprule
    Method &  Ring (n=$16$) & Torus (n=$16$)\\
    \midrule
    DAdam & $87.14 \pm 0.71$ & $87.42 \pm 0.65$\\
    HA-DSGD(Adam) & $89.29 \pm 0.48 $ &  $89.73 \pm 0.54$\\ \bottomrule
    \end{tabular}
    \caption{Top-1 test accuracy on the AGNews dataset under different topologies. The results in the table are averaged
over three random seeds. }
    \label{tab:ag_news}
\end{table}

\textbf{CIFAR10.}
We evaluate our approach on the CIFAR10 dataset \citep{Krizhevsky09learningmultiple} by training the Resnet20 model \citep{he2015deep} with Evonorm \citep{evonorm} for 300 epochs for each model. 
Following Sec. \ref{sec:extend}, we consider the extension of our algorithm to the mixing of Nesterov momentum updates, denoted by HA-DSGD (momentum) in Table \ref{tab:cifar}, and compare against the corresponding version of DSGD with momentum. We also compare against the $D^2$ algorithm \citep{Tang2018:d2} for completeness.
The results show that our approach consistently outperforms the baselines across three topologies, ring ($n=16$), torus ($n=16$), as well as the topology defined by the Davis Southern Women dataset as available in the \texttt{Networkx} library \citep{networkX2008}.
Since both the Metropolis-Hastings and the optimal spectral gap mixing schemes lead to similar results, we only compare against the Metropolis-Hastings schemes in the subsequent tasks. 

\textbf{Transformer on AG News.}\label{sec:nlp}
We evaluate the extension of our algorithm to the mixing of Adam \citep{DBLP:journals/corr/KingmaB14} updates on the NLP task of fine-tuning the \texttt{distilbert-base-uncased} model \citep{wolf2020huggingfaces} on the AGNews dataset \citep{Zhang2015CharacterlevelCN}. 
Table \ref{tab:ag_news} verifies the applicability of our approach to Adam updates.

\textbf{Imagenet.} 
To evaluate our approach on a large-scale dataset, we consider the task of training a Resnet18 model \citep{he2015deep} with evonorm on the Imagenet dataset \citep{Imagenet}.
We use a larger period of 1000 for the optimization of the mixing matrix to account for the larger number of steps per epoch. We train each model using Nesterov momentum for 90 epochs using a ring topology defined on $16$ nodes. Similar to other settings, our approach as shown in Table \ref{tab:Imagenet} outperforms DSGD, demonstrating its effectiveness under large period and dataset sizes.

\subsection{Effect of the Compression Dimension}
Proposition \ref{lem:compress} predicts that a low approximation error in the entries of the Gram matrix can be achieved through compression with dimension independent of the number of parameters and logarithmic in the number of nodes. We empirically verify this for the CIFAR10 dataset using  HA-DSGD with Nesterov momemtum and a period of $100$.
\begin{table}[t!]
    \centering
    \begin{tabular}{lc}
    \toprule
    Method & Ring (n=$16$)\\
    \midrule
    HA-DSGD(momentum, period=1000)  & $55.14 \pm 0.215$ \\
    DSGD (momentum)    & $53.22 \pm 0.25$ \\ \bottomrule
    \end{tabular}
    \caption{Top-1 Test accuracy on the Imagenet dataset, The results in the table are averaged
over three random seeds.}
    \label{tab:Imagenet}
\end{table}
In Table \ref{tab:compress}, we observe that using a sketching dimension of~$1$ leads to a significantly low performance while increasing the compression dimension to $1000$ leads to a marginal improvement in the test accuracy.


\afterpage{\FloatBarrier}

\section{Conclusion and Future Work}

In this work, we extended the analysis of DSGD to incorporate the interaction between the mixing matrix and the data heterogeneity, leading to a novel technique for dynamically adapting the mixing matrix throughout training.
We focused on a general data-dependant mixing-based analysis of the DSGD algorithm with doubly-stochastic matrices for non-convex and convex-objectives. Future work could involve extending our technique to algorithms designed for specific settings such as EXTRA \citep{EXTRA} for convex non-stochastic cases, as well as approaches based on row-stochastic, column-stochastic matrices and time-varying topologies. 
On the theoretical side, promising directions include extending our analysis to the mixing of momentum.

%


{\small
\bibliography{decentralized-mixing.bib}
\bibliographystyle{icml2022}
}

\appendix
\onecolumn


\section{Proofs of Main Results}
\subsection{Preliminaries}

We utilize the following set of standard useful inequalities:

\begin{lemma}\label{remark:lsmooth_norm}
 Let $g$ be an $L$-smooth convex function. Then we have:
	\begin{align}
	\norm{ \nabla g(\xx) - \nabla g(\yy)}_2^2 &\leq 2L \left(g(\xx)- g(\yy) - \lin{\xx-\yy,\nabla g(\yy)} \right)\,, && \forall \xx, \yy \in \R^d, 
	\end{align}
\end{lemma}

\begin{lemma}\label{remark:matvar}
Let $Y \in \R^{d \times n}$ be an arbitrary matrix and $\overline{Y}$ the matrix with each column containing the columnwise mean of $Y$ i.e. $\overline{Y} = Y\frac{\1\1^\top}{n}$. Then we have:
\begin{equation}
    \norm{Y-\overline{Y}}_F^2 = \norm{Y}_F^2 - \norm{\overline{Y}}_F^2 \leq  \norm{Y}_F^2.
\end{equation}
\end{lemma}

\begin{lemma}\label{remark:norm_of_sum}
	For arbitrary set of $n$ vectors $\{\aa_i\}_{i = 1}^n$, $\aa_i \in \R^d$
	\begin{equation}\label{eq:norm_of_sum}
	\norm{\sum_{i = 1}^n \aa_i}^2 \leq n \sum_{i = 1}^n \norm{\aa_i}^2 \,.
	\end{equation}
\end{lemma}
\begin{lemma}\label{remark:norm_of_sum_of_two}
	For given two vectors $\aa, \bb \in \R^d$ %
	\begin{align}\label{eq:norm_of_sum_of_two}
	\norm{\aa + \bb}^2 \leq (1 + \alpha)\norm{\aa}^2 + (1 + \alpha^{-1})\norm{\bb}^2,\,\, & &\forall \alpha > 0\,.
	\end{align}
	This inequality also holds for the sum of two matrices $A,B \in \R^{n \times d}$ in Frobenius norm.
\end{lemma}

\subsection{Recursion For Consensus}

The recursion for consensus, analyzed in Lemmas 9 and 12 of \citet{koloskova20:unified} relies on the following inequalities:
	\begin{align*}
	n\Xi_{t} & = \E{\norm{X^{(t)} - \overline X^{(t)}}_F^2 } = \E \norm{X^{(t)}  - \overline X^{(t-1)  }- \left(\overline X^{(t)} - \overline X^{(t-1)  } \right)}_F^2 \leq \E \norm{X^{(t)}  - \overline X^{(t-1)  }}_F^2
	\end{align*}
The above inequality, however, discards the fact that it is desirable for the update at each node to be close to the update to the mean. Our analysis below instead incorporates the effect of the mixing of the gradient through the following lemma:
\begin{lemma}\label{lem:double_mix}
The update to $X^{(t-1)}$ at the $t_{th}$ step of algorithms \ref{alg:dsgd} and \ref{alg:dsgd-gme} with mixing matrix $W^{(t-1)}$ can be reformulated as:
\begin{equation}
    X^{(t)} - \overline X^{(t)} = \rbr*{X^{(t-1)} - \overline X^{(t-1)}}W^{(t-1)} - \eta_t\rbr*{\partial F(X^{(t-1)}, \xi^{(t-1)})- \overline{\partial F}(X^{(t-1)}, \xi^{(t-1)})}W^{(t-1)}
\end{equation}
\end{lemma}

\begin{proof}
\begin{align*}
    &X^{(t)} - \overline X^{(t)} = (X^{(t-1)} - \eta_t\partial F(X^{(t-1)}, \xi^{(t-1)}))(W^{(t-1)} - \frac{1}{n}\1\1^\top)\\
    &= X^{(t-1)}W^{(t-1)} - \overline X^{(t-1)} - \eta_t \rbr*{\partial F(X^{(t-1)}, \xi^{(t-1)})W^{(t-1)} - \overline{\partial F}(X^{(t-1)}, \xi^{(t-1)})}\\
    &= \rbr*{X^{(t-1)} - \overline X^{(t-1)}}W^{(t-1)} - \eta_t\rbr*{\partial F(X^{(t-1)}, \xi^{(t-1)})- \overline{\partial F}(X^{(t-1)}, \xi^{(t-1)})}W^{(t-1)}.
\end{align*}
Where in the last step we used the identity $\frac{\1\1^\top}{n} W = \frac{\1\1^\top}{n}$, valid for any doubly stochastic matrix $W$, implying that $\overline X^{(t-1)}W^{(t-1)} = \overline X^{(t-1)}$ and $ \overline{\partial F}(X^{(t-1)}, \xi^{(t-1)})W^{(t-1)}  = \overline{\partial F}(X^{(t-1)}, \xi^{(t-1)})$.
\end{proof}

For the sake of generality and consistency with \citep{koloskova20:unified}, we prove our result under a generalization of the Assumption \ref{a:relative} on the gradient mixing error
\begin{assumption}[Relative Heterogenity with Growth]
\label{a:gme_growth}
We assume that there exist constants $\zeta'$ and $P'$, such that $\forall X \in \R^{d\times n}$:

Assumption \ref{a:relative} corresponds to a special case of the above assumption with $P'=0$.

\end{assumption}
We now prove the following consensus recursion:
\begin{lemma}\label{lem:GME_consensus}
Let $\Xi_t = \frac{1}{n}\E_t{\sum_{i=1}^n  \norm{\xx_i^{(t)}-\overline \xx^{(t)}}^2}$ denote the consensus distance at time $t$, and let $e_t = f(\overline\xx^{(t)}) - f(\xx^\star)$ for the convex case and $e_t = \norm{ \nabla f(\overline \xx^{(t)}) }_2^2$ for the non-convex case. Then:

\begin{align}
\Xi_{t} & \leq \left( 1 - \frac{p}{2}\right) \Xi_{t-1} + D \eta_{t-1}^2 e_{t-1} + A\eta_{t-1}^2. \label{eq:consensus_recursion}
\end{align}
Where $D = 36L(1-p) + 4L\frac{8-7p}{p}$ for the convex case , $\frac{8-7p}{p}P'$ for the nonconvex case and $A = \frac{8-7p}{p}({\zeta'}^2) + 3(1-p)\sigma^2$ for the non-convex case and $\frac{16-14p}{p}({\zeta'}^2) + 9(1-p)\sigma^2$ for the convex case.
\end{lemma}
\begin{proof}
\begin{align}
	& \E_{W \sim \cW} \tfrac{1}{n}\big\| \partial f( \overline X) W - \overline{\partial f}( \overline X) \big\|^2 \leq \zeta'^2 + P'\norm{{\partial f}( \overline X)}^2\,.
	\label{eq:hat-zeta-prime-withp} 
\end{align}

Let $\Xi_t = \frac{1}{n}\E_t{\sum_{i=1}^n  \norm{\xx_i^{(t)}-\overline \xx^{(t)}}^2}$ denote the consensus distance at time $t$. We have, using Lemma \ref{lem:double_mix}:
\begin{align*}
	n\Xi_{t} & = \E{\norm{X^{(t)} - \overline X^{(t)}}_F^2 } \\
	&= \E \norm{\rbr*{X^{(t-1)} - \overline X^{(t-1)}}W^{(t-1)} - \eta_t\rbr*{\partial F(X^{(t-1)}, \xi^{(t-1)})- \overline{\partial F}(X^{(t-1)}, \xi^{(t-1)})}W^{(t-1)}}_F^2\\
	&=\underbrace{\E \norm{\rbr*{X^{(t-1)} - \overline X^{(t-1)}}W^{(t-1)} - \eta_t\rbr*{\partial f(X^{(t-1)})- \overline{\partial f}(X^{(t-1)})}W^{(t-1)}}_F^2}_{=:T_1}\\ 
	&+ \underbrace{\eta_t^2\E \norm{\rbr*{\partial F(X^{(t-1)}, \xi^{(t-1)})- \overline{\partial F}(X^{(t-1)}, \xi^{(t-1)})}W^{(t-1)}- \rbr*{\partial f(X^{(t-1)})- \overline{\partial f}(X^{(t-1)})}W^{(t-1)}}^2_F}_{=:T_2}\\
\end{align*}

Where the last inequality follows from the fact that noise in the gradient is independent at each time step, and also from unbiased stochastic gradients $\E_{\xi^{(t - 1)}}\partial F(X^{(t-1)}, \xi^{(t-1)}) = \partial f(X^{(t - 1)})$. 
We first observe that, using assumption \ref{a:p}, we have:
\begin{align*}
    &\E \norm{\rbr*{\partial F(X^{(t-1)}, \xi^{(t-1)})- \overline{\partial F}(X^{(t-1)}, \xi^{(t-1)})}W^{(t-1)}- \rbr*{\partial f(X^{(t-1)})- \overline{\partial f}(X^{(t-1)})}W^{(t-1)}}^2_F\\
    &\leq
    (1-p)\E \norm{\rbr*{\partial F(X^{(t-1)}, \xi^{(t-1)})- \overline{\partial F}(X^{(t-1)}, \xi^{(t-1)})}- \rbr*{\partial f(X^{(t-1)})- \overline{\partial f}(X^{(t-1)})}}^2_F
\end{align*}
Furthermore, using \eqref{remark:matvar}, we have:
\begin{align*}
    \E \norm{\rbr*{\partial F(X^{(t-1)}, \xi^{(t-1)})- \overline{\partial F}(X^{(t-1)}, \xi^{(t-1)})}- \rbr*{\partial f(X^{(t-1)})- \overline{\partial f}(X^{(t-1)})}}^2_F \leq \E \norm{\partial F(X^{(t-1)}, \xi^{(t-1)})- \partial f(X^{(t-1)})}^2_F
\end{align*}
We then add and subtract the gradients at the mean point $\partial F(\overline{X}^{(t-1)}, \xi^{(t-1)})$ and the corresponding mean $\partial F(\overline{X}^{(t-1)})$ to obtain:
\begin{align*}
    &\E \norm{\rbr*{\partial F(X^{(t-1)}, \xi^{(t-1)})- \partial f(X^{(t-1)})}}^2_F\\
    &\stackrel{\text{Lemma} \ \ref{remark:norm_of_sum}}{\leq~} 3 \norm{\partial F(X^{(t-1)}, \xi^{(t-1)}) - \partial F(\overline{X}^{(t-1)}, \xi^{(t-1)})}^2+ 3 \E \norm{\partial f(X^{(t-1)}) - \partial F(\overline{X}^{(t-1)}}^2\\ 
    &+ 3 \E\norm{\partial F(\overline{X}^{(t-1)}, \xi^{(t-1)})-\partial F(\overline{X}^{(t-1)})}^2_F
\end{align*}

Using the $L$ smoothness of each node's objective (Assumption \ref{a:lsmooth}), the first two terms can be bounded as follows:

\begin{align*}
   \E \norm{\partial F(X^{(t-1)}, \xi^{(t-1)}) - \partial F(\overline{X}^{(t-1)}, \xi^{(t-1)})}^2
    &\leq L^2\E \norm{X^{(t-1)} - \overline{X}^{(t-1)}}_F^2
\end{align*}
Similarly, we have:

\begin{equation}\label{eq:smooth_bound}
\begin{split}
    \E \norm{\partial f(X^{(t-1)}) - \partial F(\overline{X}^{(t-1)})}^2
    & \leq L^2\E \norm{X^{(t-1)} - \overline{X}^{(t-1)}}_F^2.
\end{split}
\end{equation}

Subsequently, we utilize the assumptions on stochasticity \ref{a:stoc} to bound the third term. 

We proceed separately for the Convex and non-convex cases:

\textbf{Convex Case}: We add and subtract $\partial F(X^\star, \xi^{(j)})$ and the corresponding mean $\partial F(X^\star)$ to obtain:
\begin{equation}\label{eq:stoch_convex}
    \begin{split}
    &\E \norm{\partial F(\overline{X}^{(t-1)}, \xi^{(t-1)})-\partial F(\overline{X}^{(t-1)})}^2_F\\
    &= \E \norm{\rbr*{\partial F(\overline{X}^{(t-1)}, \xi^{(t-1)})-\partial F(X^\star, \xi^{(j)})}-\rbr*{\partial F(\overline{X}^{(t-1)}-\partial F(X^\star)} +\rbr*{\partial F(X^\star, \xi^{(j)})-\partial F(X^\star))}}^2_F\\
    &\stackrel{\text{Lemma} \ \ref{remark:norm_of_sum}}{\leq~} 3\E \norm{\partial F(\overline{X}^{(t-1)}, \xi^{(t-1)}) -\partial F(X^\star, \xi^{(j)})}^2_F +3\norm{\partial F(\overline{X}^{(t-1)})-\partial F(X^\star)} +3\norm{\partial F(X^\star, \xi^{(j)})-\partial F(X^\star))}^2_F\\
    &\stackrel{\text{Lemma} \ \ref{remark:lsmooth_norm}}{\leq} 3\cdot 2Ln(f(\xx)-f(\xx^\star)) + 3 \cdot 2Ln(f(\xx)-f(\xx^\star)) +3n\overline{\sigma}^2\\
    &= 12Ln(f(\xx)-f(\xx^\star)) + 3n\overline{\sigma}^2
     \end{split}
\end{equation}

\textbf{Non-convex Case}: We directly utilize the uniform bound on the stochasticity (assumption \ref{a:stoc}) to obtain:
\begin{equation}\label{eq:stoch_nonconvex}
    \E \norm{\partial F(\overline{X}^{(t-1)}, \xi^{(t-1)})-\partial F(\overline{X}^{(t-1)})}^2_F \leq n\hat{\sigma}^2
\end{equation}

The final bound on $T_2$ is therefore given by:

\textbf{Convex case}:
\begin{align*}
T_2  \leq \eta_t^26(1-p)L^2\E \norm{\rbr*{X^{(t-1)} - \overline{X}^{(t-1)}}}_F^2 + 36(1-p)\eta_t^2Ln(f(\xx^{(t-1})-f(\xx^\star)) + 9n(1-p)\eta_t^2\overline{\sigma}^2
\end{align*}

\textbf{Non-convex case}:
\begin{align*}
T_2  \leq 6(1-p)\eta_t^2 L^2\E \norm{\rbr*{X^{(t-1)} - \overline{X}^{(t-1)}}}_F^2 + 3n(1-p)\eta_t^2\overline{\sigma}^2
\end{align*}
We now bound $T_1$ as follows: 
\begin{align*}
&\E \norm{\rbr*{X^{(t-1)} - \overline X^{(t-1)}}W^{(t-1)} - \eta_t\rbr*{\partial f(X^{(t-1)})- \overline{\partial f}(X^{(t-1)})}W^{(t-1)}}_F^2\\ 
&= \E \biggl\|\rbr*{X^{(t-1)} - \overline{X}^{(t-1)}}W^{(t-1)} - \eta_t\rbr*{\partial F(\overline{X}^{(t-1)})- \overline{\partial F}(\overline{X}^{(t-1)})}W^{(t-1)} \\-&\eta_t\rbr*{\rbr*{\partial f(X^{(t-1)}) - \partial F(\overline{X}^{(t-1)})} - \rbr*{\overline{\partial f}(X^{(t-1)})-\overline{\partial F}(\overline{X}^{(t-1)})}}W^{(t-1)}\biggr\|_F^2\\
	  \stackrel{\text{Lemma} \ \ref{remark:norm_of_sum_of_two}}{\leq~} &(1+\beta_1)\E \biggl\|\rbr*{X^{(t-1)} - \overline{X}^{(t-1)}}W^{(t-1)} \\
	 -&\eta_t\rbr*{\rbr*{\partial f(X^{(t-1)}) - \partial F(\overline{X}^{(t-1)})} - \rbr*{\overline{\partial f}(X^{(t-1)})-\overline{\partial F}(\overline{X}^{(t-1)})}}W^{(t-1)}\biggr\|_F^2\\
	 &+ (1+\beta_1^{-1})\E \norm{ - \eta_t\rbr*{\partial F(\overline{X}^{(t-1)})- \overline{\partial F}(\overline{X}^{(t-1)})}W^{(t-1)}}_F^2\\
	 \stackrel{\text{Lemma} \ \ref{remark:norm_of_sum_of_two} \ , Assumption \ \ref{a:p}}{\leq~}& (1-p)(1+\beta_1)(1+\beta_2)\E \norm{\rbr*{X^{(t-1)} - \overline{X}^{(t-1)}}}_F^2 \\
	 +&\eta_t^2(1-p)(1+\beta_1)(1+\beta_2^{-1})\E \biggl\|\rbr*{\partial f(X^{(t-1)}) - \partial F(\overline{X}^{(t-1)})}\\ &- \rbr*{\overline{\partial f}(X^{(t-1)})-\overline{\partial F}(\overline{X}^{(t-1)})}\biggr\|_F^2\\
	 &+ (1+\beta_1^{-1})\E \norm{\eta_t\rbr*{\partial F(\overline{X}^{(t-1)})- \overline{\partial F}(\overline{X}^{(t-1)})}W^{(t-1)}}_F^2 
\end{align*}
The second term can be bounded by utilizing Equation \ref{eq:smooth_bound} and Equation \ref{remark:matvar} as follows:

\begin{align*}
    \E \biggl\|\rbr*{\partial f(X^{(t-1)}) - \partial F(\overline{X}^{(t-1)})}- \rbr*{\overline{\partial f}(X^{(t-1)})-\overline{\partial F}(\overline{X}^{(t-1)})}\biggr\|_F^2 &\stackrel{\text{Lemma}
    \ \ref{remark:matvar}}{\leq}  \E \biggl\|\rbr*{\partial f(X^{(t-1)}) - \partial F(\overline{X}^{(t-1)})}\biggr\|_F^2\\ 
    &\stackrel{\ref{eq:smooth_bound}}{\leq} L^2\norm{X^{(t-1)} - \overline{X}^{(t-1)}}_F^2.
\end{align*}
Therefore, we obtain:
\begin{align*}
    T_1 &\leq \rbr*{(1-p)(1+\beta_1)(1+\beta_2)+ \eta^2_t(1-p)(1+\beta_1)(1+\beta_2^{-1})L^2}\norm{X^{(t-1)} - \overline{X}^{(t-1)}}_F^2 \\
    &\qquad \qquad + (1+\beta_1^{-1})\eta_t^2\E \norm{\rbr*{\partial F(\overline{X}^{(t-1)})- \overline{\partial F}(\overline{X}^{(t-1)})}W^{(t-1)}}_F^2
\end{align*}
Finally, incorporating the bound on $T_2$, we obtain:

\textbf{Convex Case}:
\begin{align*}
    n\Xi_{t} &\leq \rbr*{(1-p)(1+\beta_1)(1+\beta_2^{-1})+ \eta^2_t(1-p)(1+\beta_1)(1+\beta_2^{-1})}\norm{X^{(t-1)} - \overline{X}^{(t-1)}}_F^2 \\
    & \qquad + (1+\beta_1^{-1})\eta_t^2\E \norm{\rbr*{\partial F(\overline{X}^{(t-1)})- \overline{\partial F}(\overline{X}^{(t-1)})}W^{(t-1)}}_F^2 + 6(1-p)L^2\E \norm{\rbr*{X^{(t-1)} - \overline{X}^{(t-1)}}}_F^2 \\
    & \qquad + 36(1-p)\eta_t^2Ln(f(\xx)-f(\xx^\star)) + 9n(1-p)\eta_t^2\overline{\sigma}^2
\end{align*}
\textbf{Nonconvex Case}:

\begin{align*}
n\Xi_{t} &\leq \rbr*{(1-p)(1+\beta_1)(1+\beta_2^{-1})+ \eta^2_t(1-p)(1+\beta_1)(1+\beta_2^{-1})}\norm{\rbr*{X^{(t-1)} - \overline{X}^{(t-1)}}}_F^2\\
& \qquad + (1+\beta_1^{-1})\eta_t^2\E \norm{\rbr*{\partial F(\overline{X}^{(t-1)})- \overline{\partial F}(\overline{X}^{(t-1)})}W^{(t-1)}}_F^2 + 6(1-p)L^2\eta_t^2\E \norm{\rbr*{X^{(t-1)} - \overline{X}^{(t-1)}}}_F^2 + 3n(1-p)\eta_t^2\overline{\sigma}^2
\end{align*}
We now choose $\beta_1$ such that $(1-p)(1+\beta_1) = (1-\frac{7p}{8})$ i.e. $\beta_1 = \frac{p}{8(1-p)}$ Subsequently, we choose $\beta_2$ such that $((1-\frac{7p}{8})(1+\beta_2) = (1-\frac{3p}{4})$ i.e. $\beta_2 = \frac{p}{8-7p}$.
Then, assuming that the step size $\eta_t$ satisfies, $\eta_t^2 \leq \frac{\frac{p}{4}}{(1-p)(1+\beta_1)(1+\beta_2^{-1})L^2 + 6(1-p)L^2} = \frac{\frac{p}{4}}{\rbr*{(1-\frac{7p}{8})\frac{8-6p}{p}+6(1-p)}L^2}$, we obtain:
\begin{align*}
	n\Xi_{t} \leq & (1-\frac{3p}{4})\E \norm{\rbr*{X^{(t-1)} - \overline{X}^{(t-1)}}}_F^2 + \frac{p}{4}\E \norm{\rbr*{X^{(t-1)} - \overline{X}^{(t-1)}}}_F^2\\
	 &+ (1+\beta_1^{-1}) \eta_t^2\E \norm{\rbr*{\partial F(\overline{X}^{(t-1)})- \overline{\partial F}(\overline{X}^{(t-1)}}W^{(t-1)}}_F^2 + 6(1-p)L^2\eta_t^2\E \norm{\rbr*{X^{(t-1)} - \overline{X}^{(t-1)}}}_F^2 + 3n(1-p)\eta_t^2\overline{\sigma}^2\\
	 &\leq (1-\frac{p}{2})\E \norm{\rbr*{X^{(t-1)} - \overline{X}^{(t-1)}}}_F^2\\
	 &+ \eta_t^2(1+\beta_1^{-1})\E \norm{ \rbr*{\partial F(\overline{X}^{(t-1)})- \overline{\partial F}(\overline{X}^{(t-1)})}W^{(t-1)}}_F^2+3n(1-p)\eta_t^2\overline{\sigma}^2
\end{align*}

Since $\frac{\frac{p}{4}}{\rbr*{(1-\frac{7p}{8})\frac{8-6p}{p}+6(1-p)}L^2} \geq \frac{p^2}{80L^2}$, we only require the step size to be $\cO(\frac{p^2}{L^2})$, same as \citet{koloskova20:unified}.
Thus the consensus distance decreases linearly, along with an error dependent on the diffusion of the gradients across nodes. Finally, substituting the assumption \ref{a:gme_growth}  for the non-convex case, we obtain:
\begin{align*}
	n\Xi_{t} \leq (1-\frac{p}{2})\E \norm{\rbr*{X^{(t-1)} - \overline{X}^{(t-1)}}}_F^2+ \eta_t^2\frac{8-7p}{p}({\zeta'}^2 + P' \norm{ \overline{\partial F}(\overline{X}) }^2)\\
	= (1-\frac{p}{2})\E \norm{\rbr*{X^{(t-1)} - \overline{X}^{(t-1)}}}_F^2+ \eta_t^2(1+\beta_1^{-1}){\zeta'}^2 + \eta_t^2\frac{8-7p}{p} (1-p)P' \norm{ \overline{\partial F}(\overline{X})}^2 + 3n(1-p)\eta_t^2\sigma^2.
\end{align*}
For the convex case, we first bound the gradient mixing error at $X$ in terms of that at $X^*$ 
as follows:

\begin{align*}
    &\E \norm{ \rbr*{\partial F(\overline{X}^{(t-1)})- \overline{\partial F}(\overline{X}^{(t-1)})}W^{(t-1)}}_F^2\\ &= \E \norm{ \rbr*{\partial F(\overline{X}^{(t-1)})- \partial F(X^*)  - (\overline{\partial F}(\overline{X}^{(t-1)}) - \overline{\partial F}(X^*))}W^{(t-1)} + \rbr*{\partial F(X^*)- \overline{\partial F}(X^*)}W^{(t-1)}}_F^2\\ 
     & \stackrel{\text{Lemma} \ \ref{remark:norm_of_sum}}{\leq~} 2 \E \norm{\rbr*{\partial F(\overline{X}^{(t-1)})- \partial F(X^*)  - (\overline{\partial F}(\overline{X}^{(t-1)}) - \overline{\partial F}(X^*))}W^{(t-1)}}^2_2 + 2  \E \norm{\rbr*{\partial F(X^*)- \overline{\partial F}(X^*)}W^{(t-1)}}_F^2\\
     & \stackrel{\text{Assumption} \ \ref{a:p}}{\leq~} 2(1-p)\E \norm{\partial F(\overline{X}^{(t-1)})- \partial F(X^*)  - (\overline{\partial F}(\overline{X}^{(t-1)}) - \overline{\partial F}(X^*))}^2_2 + 2  \E \norm{\rbr*{\partial F(X^*)- \overline{\partial F}(X^*)}W^{(t-1)}}_F^2\\
     & \stackrel{\text{Lemma} \ \ref{remark:matvar}}{\leq~} 2 (1-p) \E \norm{\partial F(\overline{X}^{(t-1)})- \partial F(X^*)  }^2_2 + 2  \E \norm{\rbr*{\partial F(X^*)- \overline{\partial F}(X^*)}W^{(t-1)}}_F^2\\
    &\leq 2(1-p) \E \norm{\partial F(\overline{X}^{(t-1)})- \partial F(X^*)}^2_2 + 2  \E \norm{\rbr*{\partial F(X^*)- \overline{\partial F}(X^*)}W^{(t-1)}}_F^2\\
    &\leq 4(1-p) L \E\left( f(\overline{\xx}^{(t-1)}) - f(\xx^\star)\right) + 2  \E \norm{\rbr*{\partial F(X^*)- \overline{\partial F}(X^*)}W^{(t-1)}}_F^2.
\end{align*}

Where in the last step, we used Equation \ref{remark:lsmooth_norm}. Therefore, we obtain:

\begin{align*}
	n\Xi_{t} 
	&\leq (1-\frac{p}{2})\E \norm{\rbr*{X^{(t-1)} - \overline{X}^{(t-1)}}}_F^2\\ &+ 4(1-p)\eta_t^2\frac{8-7p}{p} L \E\left( f(\overline{\xx}^{(j)}) - f(\xx^\star)\right) +2 \eta_t^2(1+\beta_1^{-1})n \overline\zeta^2 + 36Ln(1-p)\eta_t^2(f(\xx)-f(\xx*)) + 9n(1-p)\eta_t^2\overline{\sigma}^2
\end{align*}

\end{proof}

\subsection{Convergence Rate}

We utilize the consensus recursion in Lemma \ref{lem:GME_consensus} to bound an appropriately weighted sum of the consensus iterates as follows:
\begin{align*}
    \sum_{t=0}^T w_t n\Xi_{t} \leq \sum_{t=1}^T w_t(1-\frac{p}{2})n\Xi_{t-1} + \sum_{t=1}^T w_t\eta_{t-1}^2De_{t-1} + \sum_{t=1}^T w_t \eta_{t-1}^2A
\end{align*}

Recursively substituting $n\Xi_{t-1}$ for $t$ in $[1,\cdots,T]$, we then obtain:
\begin{align*}
    \sum_{t=0}^T w_t n\Xi_{t} &\leq \sum_{t=1}^T\sum_{j=0}^{t-1} w_t \eta_j^2(1-\frac{p}{2})^{t-j-1}(De_j + A)\\
    &= \sum_{t=1}^{T} \sum_{j=0}^{t-1} w_t \eta_j^2  (1-\frac{p}{2})^{t-j-1}(De_j + A)\\
    &= \sum_{j=0}^{T-1}\sum_{t=j+1}^T \eta_j^2 w_t (1-\frac{p}{2})^{t-j-1}(De_j + A)\\
    &\leq \sum_{j=0}^T\sum_{t=j+1}^\infty \eta_j^2 w_t (1-\frac{p}{2})^{t-j-1}(De_{t-1} + A)\\
    &\leq \sum_{j=0}^T \eta_j^2\frac{2}{p} w_j(De_j + A)\ .
\end{align*}

Where in the last step we used $w_t \leq w_j$ for $j\geq t$

We thus obtain an Equation having the same form as Equation 18 of \citep{koloskova20:unified} :

\begin{align}\label{eq:rec2}
B \cdot \sum_{t = 0}^T w_t \Xi_{t} \ \leq\  \frac{b}{2} \cdot \sum_{t = 0}^T w_t e_t + A B \frac{2}{p}\cdot \sum_{t = 0}^T w_t \eta_t^2  \,,
\end{align}

where $\eta$ satisfies $\eta \leq \sqrt{\frac{pbD}{2B}}$ and the factor $B$ is as defined in \citep{KoloskovaLSJ19decentralized} for the different cases.



The rest of the proof involves utilizing the descent lemma in \citep{koloskova20:unified} and choosing the appropriate step size following exactly the use of Equation 18 in \citep{koloskova20:unified}. Finally, setting $P'=0$ leads to the convergence rates provided in Theorem \ref{thm:rate}.

\subsection{Proof of Proposition \ref{lem:period_effect}}

We first note that
\begin{equation}\label{eq:update_decomp}
    \overline{X}^{(t+H)} - \overline{X}^{(t)} =  \sum_{i=0}^{H-1} -\eta_{t+i} \partial  f(X^{(t+i)})
\end{equation}

We further have:
\begin{align*}
   \norm{\rbr*{\partial f\rbr*{\overline{X}^{(t+H)}}-\overline{\partial f}\rbr*{\overline{X}^{(t+H)}}}W^{(t)}}^2_F &\stackrel{\text{Lemma} \ \ref{remark:norm_of_sum}}{\leq~} 2\norm{\rbr*{\partial f\rbr*{\overline{X}^{(t)}}-\overline{\partial f}\rbr*{\overline{X}^{(t)}}}W^{(t)}}^2_F\\ &+ 2\norm{\rbr*{\partial f\rbr*{\overline{X}^{(t+H)}}-\overline{\partial f}\rbr*{\overline{X}^{(t+H)}} -\rbr*{\partial f\rbr*{\overline{X}^{(t)}}-\overline{\partial f}\rbr*{\overline{X}^{(t)}}}}W^{(t)}}^2_F
\end{align*}

Applying Lemma \ref{remark:matvar} to the second term in the RHS yields:

\begin{align*}
   \norm{\rbr*{\partial f\rbr*{\overline{X}^{(t+H)}}-\overline{\partial f}\rbr*{\overline{X}^{(t+H)}}}W^{(t)}}^2_F &\stackrel{Lemma \ref{remark:matvar}}{\leq~}  2\norm{\rbr*{\partial f\rbr*{\overline{X}^{(t)}}-\overline{\partial f}\rbr*{\overline{X}^{(t)}}}W^{(t)}}^2_F + 2\norm{\rbr*{\partial f\rbr*{\overline{X}^{(t+H)}}-\partial f\rbr*{\overline{X}^{(t)}}}W^{(t)}}^2_F\\
\end{align*}
Finally, using Equation \ref{eq:update_decomp} and the $L$-smoothness of the objectives (Assumption \ref{a:lsmooth}), we obtain:
\begin{align*}
     \norm{\rbr*{\partial f\rbr*{\overline{X}^{(t+H)}}-\overline{\partial f}\rbr*{\overline{X}^{(t+H)}}}W^{(t)}}^2_F &\leq 2\norm{\rbr*{\partial f\rbr*{\overline{X}^{(t)}}-\overline{\partial f}\rbr*{\overline{X}^{(t)}}}W^{(t)}}^2_F + 2L^2\norm{ \overline{X}^{(t+H)} - \overline{X}^{(t)}}^2\\ &\leq 2\norm{\rbr*{\partial f\rbr*{\overline{X}^{(t)}}-\overline{\partial f}\rbr*{\overline{X}^{(t)}}}W^{(t)}}^2_F + 2H\sum_{i=0}^{H-1}\eta_t^2L^2\norm{\partial  f(X^{(t+i)})}^2_F\ 
\end{align*}
\subsection{Proof of Proposition \ref{prop:local}}

We start by adding and subtracting the corresponding gradients at the mean parameters $X$:

\begin{align*}
    \norm{\rbr*{\partial f(\overline X) - \overline{\partial f}(\overline X)}W}_F^2 &= \norm{\rbr*{\partial f(\overline X) - \partial f(X) - \rbr*{\overline{\partial f}(\overline X) - \overline{\partial f}(X)}}W + \rbr*{\partial f(X) - \overline{\partial f}(X)}W}_F^2\\
    &\stackrel{\text{Lemma} \ \ref{remark:norm_of_sum}}{\leq~} \norm{\rbr*{\partial f(\overline X) - \partial f(X) - \rbr*{\overline{\partial f}(\overline X) - \overline{\partial f}(X)}}}_F^2 + 2\norm{\rbr*{\partial f(X) - \overline{\partial f}(X)}W}_F^2\\
    &\stackrel{\text{Lemma} \ \ref{remark:matvar}}{\leq~} 2 \norm{\rbr*{\partial f(\overline X) - \partial f(X)}}_F^2 + 2\norm{\rbr*{\partial f(X) - \overline{\partial f}(X)}W}_F^2\\
    &\stackrel{\text{Lemma} \ \ref{a:lsmooth}}{\leq~} L^2\norm{X-\overline X}_F^2 + 2\norm{\rbr*{\partial f(X) - \overline{\partial f}(X)}W}_F^2,
\end{align*}

\subsection{Spectral Norm of Doubly Stochastic Matrices with Non-negative Entries}

\begin{proposition}\label{prop:perron}
Let $W \in \R^{n \times n}$ be possibly asymmetric doubly stochastic matrix with non-negative entries. Then the spectral norm $\norm{W}_2$ is bounded by $1$.
\end{proposition}
 
 \begin{proof}
 We note that $W^TW$ is itself a symmetric doubly-stochastic matrix and therefore has an eigenvector $\frac{1}{\sqrt{n}}\1$ with eigenvalue $1$. Perron-Frobenius theorem then implies that the largest eigenvalue of $ (W^{(t)})^{\top} W^{(t)}$ is bounded by $1$, completing the proof.
 \end{proof}

\subsection{Proof of Proposition \ref{prop:stoch}}

The proof proceeds by introducing the stochastic gradients into the LHS as follows:
\begin{align*}
&\Ea{\norm{\rbr*{\partial f(\overline{X})- \overline{\partial f}(\overline{X})}W^{*}(\xi)}_F^2} = \Ea{\norm{\rbr*{\partial f(\overline{X})- \overline{\partial f}(\overline{X}) - \rbr*{\partial f(\overline{X},\xi)- \overline{\partial f}(\overline{X},\xi)} + \rbr*{\partial f(\overline{X},\xi)- \overline{\partial f}(\overline{X},\xi)}}W^{*}(\xi)}_F^2}
\\& \stackrel{\text{Lemma} \ \ref{remark:norm_of_sum}}{\leq~}2\Ea{\norm{\rbr*{\partial f(\overline{X},\xi)- \overline{\partial f}(\overline{X},\xi)}W^{*}(\xi)}_F^2} + 2\Ea{\norm{\rbr*{\partial f(\overline{X})- \overline{\partial f}(\overline{X}) - \rbr*{\partial f(\overline{X},\xi)- \overline{\partial f}(\overline{X},\xi)}}W^{*}(\xi)}^2_2}
\\& \stackrel{\text{Lemma} \ \ref{remark:matvar}}{\leq~} 2\Ea{\norm{\rbr*{\partial f(\overline{X},\xi)- \overline{\partial f}(\overline{X},\xi)}W^{*}(\xi)}_F^2} + 2\Ea{\norm{\rbr*{\partial f(\overline{X}) - \partial f(\overline{X},\xi)}W^{*}(\xi)}^2_2}.
\end{align*}

Since $W^{*}(\xi)$ is doubly-stochastic, using Proposition \ref{prop:perron}, we obtain a bound on the spectral norm $\norm{W^{(*)}}_2 \leq 1$. Combining the bound on the spectral norm with the assumption on the variance yields:
\begin{align*}
\Ea{\norm{\rbr*{\partial f(\overline{X})- \overline{\partial f}(\overline{X})}W^{*}(\xi)}_F^2} &\leq 2\Ea{\norm{\rbr*{\partial f(\overline{X},\xi)- \overline{\partial f}(\overline{X},\xi)}W^{*}(\xi)}_F^2} + 2\Ea{\norm{\rbr*{\partial f(\overline{X}) - \partial f(\overline{X},\xi)}}^2_2}\\ &\leq 2\Ea{\norm{\rbr*{\partial f(\overline{X},\xi)- \overline{\partial f}(\overline{X},\xi)}W^{*}(\xi)}_F^2} + 2\overline{\sigma^2}
\end{align*}
\subsection{Proof of Proposition \ref{lem:compress}}
We utilize the following compression bound, that arises as a consequence of the concentration of $\chi^2$ random variables, as often utilized in the proof of the Johnson–Lindenstrauss lemma \citep{boucheron:hal-00794821}:
\begin{lemma}\label{lem:JL}
 Let $\{\uu_1,\cdots,\uu_m\} \in \mathbb{R}^{d}$. Assume that the entries in $A \subset \mathbb{R}^{k \times d}$ are sampled independently from $\mathcal{N}(0,1)$. Then, for $k \geq 100(\frac{log(\frac{m}{\delta})}{\epsilon^2})$, with probability greater than $1-\delta$ , we have,$\forall i,j \in [m]$:
 \begin{equation}\label{eq:norm_bound}
     (1-\epsilon)\norm{\uu_i-\uu_j}^2  \leq \frac{1}{k}\norm{A\uu_i-A\uu_j}^2 \leq (1+\epsilon)\norm{\uu_i-\uu_j}^2 
 \end{equation}
\end{lemma}

Slightly weaker bounds can be obtained in more general settings such as that of sub-Gaussian random variables but we restrict to the Gaussian case in the theory as well as implementations of our algorithm.

Now, adding $\{-\uu_1,\cdots,-\uu_m\}$ to the set of points and applying Lemma \ref{lem:JL} yields, $\forall i, j \in [m]$:
\begin{equation}\
     (1-\epsilon)\norm{\uu_i-\uu_j}^2  \leq \norm{A\uu_i\pm A\uu_j}^2 \leq (1+\epsilon)\norm{\uu_i-\uu_j}^2 
 \end{equation}
Therefore, we bound the inner product as follows:
\begin{align*}
   \frac{1}{k}\lin{A\uu_i,A\uu_j} &= \frac{1}{4k}\rbr*{\norm{A\uu_i+ A\uu_j}^2 - \norm{A\uu_i- A\uu_j}^2} \leq \frac{1}{4}\rbr*{(1+\epsilon)\norm{\uu_i+ \uu_j}^2-(1-\epsilon)\norm{\uu_i-\uu_j}^2}\\
   &\leq \lin{\uu_i,\uu_j} + \frac{1}{2}\epsilon \rbr*{\norm{\uu_i+ \uu_j}^2+\norm{\uu_i- \uu_j}^2} \leq \lin{\uu_i,\uu_j} + \epsilon\max_i \norm{\uu_i}^2\\
\end{align*}
Similarly, we obtain the lower bound:
\begin{align*}
   \lin{\uu_i,\uu_j} - \epsilon\max_i \norm{\uu_i}^2 \leq \frac{1}{k}\lin{A\uu_i,A\uu_j} 
\end{align*}

\section{Using different matrices for Parameter and Gradient Mixing}

An additional advantage of our analysis is that it decouples the effect of parameter and gradient mixing. This allows our analysis to be extended to the case of use of different mixing matrices $W_p$ and $W_g$ for mixing the parameters and gradients at each step respectively. Concretely, we consider the following algorithm:

\begin{algorithm}[H]
\let\oldendfor\algorithmicendfor
\renewcommand{\algorithmicendfor}{\algorithmicend\ \textbf{parallel for}}
\let\olddo\algorithmicdo
\renewcommand{\algorithmicdo}{\textbf{do in parallel on all workers}}
    \begin{minipage}{\linewidth}
	\caption{\textsc{Decentralized SGD with decoupled mixing}}\label{alg:dsgd_double}
	\begin{algorithmic}[1]
		\INPUT{$X^{(0)}$, stepsizes $\{\eta_t\}_{t=0}^{T-1}$, number of iterations $T$, mixing matrix distributions $\cW_p^{(t)},\cW_g^{(t)}$, $t \in [0, T]$} 
		\FOR{$t$\textbf{ in} $0\dots T$}
		\STATE $G^{(t)} = \partial F(X^{(t)}, \xi^{(t)})$ \hfill $\triangleright$ stochastic gradients
		\STATE $W_p^{(t)} \sim \cW_p^{(t)}, W_g^{(t)} \sim \cW_g^{(t)}$ \hfill  $\triangleright$ sample mixing matrices
		\STATE $X^{(t + 1)} = X^{(t)}W_p^{(t)} -\eta_t G^{(t)} W_g^{(t)}$ \hfill  $\triangleright$ update \& mixing
		\ENDFOR
	\end{algorithmic}
	\end{minipage}
\let\algorithmicendfor\oldendfor
\let\algorithmicdo\olddo
\end{algorithm}


We now show that the above algorithm leads to convergence rates having the same dependence on $p$ and $\zeta'$ as Theorem \ref{thm:rate} but with these parameters defined as above in terms of $\cW_u^{(t)}$ and $\cW_g^{(t)}$. For instance, for the Non-convex case, we obtain that $\frac{1}{T + 1}\sum_{t = 0}^T \E \norm{\nabla f(\overline{\xx}^{(t)})}_2^2 \leq \epsilon$ after
\begin{align*}
 \cO \left( \frac{ \sigma^2}{n \epsilon^2}  +  \frac{ \zeta' + \sigma \sqrt{p}   }{ p \epsilon^{3/2}}   + \frac{1}{p \epsilon}  \right) \cdot L F_0
\end{align*}
iterations. Analogously, we obtain the corresponding convergence rates for the convex case with $\zeta_{\star}'$ defined at the optimum i.e. $\E_{W_g \sim \cW_g^{(t)}} \tfrac{1}{n}\big\| \partial f(X_{\star}) W_g - \overline{\partial f}( X_{\star}) \big\|^2 \leq  \zeta'^2\,$.
Similar to Lemma \ref{lem:double_mix}, the update can then be expressed as
\begin{equation}
    X^{(t)} - \overline X^{(t)} = \rbr*{X^{(t-1)} - \overline X^{(t-1)}}W_p^{(t-1)} - \eta_t\rbr*{\partial F(X^{(t-1)}, \xi^{(t-1)})- \overline{\partial F}(X^{(t-1)}, \xi^{(t-1)})}W_g
\end{equation}

Subsequently, analogous to the proof of Theorem \ref{thm:rate}, we obtain the following decomposition of the consensus iterates:

\begin{align*}
	n\Xi_{t} & =\underbrace{\E \norm{\rbr*{X^{(t-1)} - \overline X^{(t-1)}}W_u^{(t-1)} - \eta_t\rbr*{\partial f(X^{(t-1)})- \overline{\partial f}(X^{(t-1)})}W_g^{(t-1)}}_F^2}_{=:T_1}\\ 
	&+ \underbrace{\eta_t^2\E \norm{\rbr*{\rbr*{\partial F(X^{(t-1)}, \xi^{(t-1)})- \overline{\partial F}(X^{(t-1)}, \xi^{(t-1)})}- \rbr*{\partial f(X^{(t-1)})- \overline{\partial f}(X^{(t-1)})}}W_g^{(t-1)}}^2_F}_{=:T_2}
\end{align*}

Now, for $p$ and $\zeta'$ satisfying:

\begin{align}
    \E_{W_u\sim \cW_u^{(t)}} \big\| X W_u - \overline X \big\| _F^2 \leq (1-p) \big\| X - \overline X \big\|_F^2 \,,
\end{align}
and,
\begin{align}
	&\E_{W_g \sim \cW_g^{(t)}} \tfrac{1}{n}\big\| \partial f( X) W_g - \overline{\partial f}( X) \big\|^2 \leq  \zeta'^2\,,
\end{align}
we obtain the analogous consensus recursion:
\begin{align}
\Xi_{t} & \leq \left( 1 - \frac{p}{2}\right) \Xi_{t-1} + D \eta_{t-1}^2 e_{t-1} + A\eta_{t-1}^2. \label{eq:consensus_recursion},
\end{align}
where $D = 36L+ 4L\frac{8-7p}{p}$ for the convex case , $\frac{8-7p}{p}P'$ for the nonconvex case and $A = \frac{8-7p}{p}({\zeta'}^2) + 3\sigma^2$ for the non-convex case and $\frac{16-14p}{p}({\zeta'}^2) + 9\sigma^2$ for the convex case.

\textbf{D-cliques \citep{bellet2021dcliques}}: Suppose that the graph can be divided into $K$ cliques, such that the mean gradient for each clique equals the mean across the entire graph. Let the nodes be numbered such that the $n_{k}$ nodes belonging to the $k_{th}$ clique succeed the $n_{k-1}$ nodes belonging to the ${(k-1)}_{th}$ clique. Then, we observe that utilizing a block matrix of the type
    \[\left(\begin{matrix}
    \frac{1}{n_1}\1\1^\top & \mathbf{0} & \cdots & \mathbf{0} \\
    \mathbf{0} & \frac{1}{n_2}\1\1^\top & \cdots & \mathbf{0} \\
    \vdots & & & \vdots \\
    \mathbf{0} & \mathbf{0} & \cdots & \frac{1}{n_K}\1\1^\top
  \end{matrix}\right)\]
leads to zero Gradient Mixing Error. 
This corresponds to the proposed algorithm in D-cliques \citep{bellet2021dcliques} where $W_g^{(t)}$ is set to a matrix performing uniform averaging within each clique, while $W_u^{(t)}$ utilizes all the edges for mixing. For unbiased cliques, we obtain $\zeta' = 0$. Therefore, our analysis above provides an explanation for the improvements achieved by decoupled parameter mixing and clique-averaging \citep{bellet2021dcliques} under the presence of unbiased cliques. We further note that, unlike the algorithm presented in \citep{bellet2021dcliques}, our algorithm HA-DSGD with random sampling of mixing matrices does not involve the additional communication overhead for separately mixing the gradients at each time-step.
\section{Mixing Error under Permutations}

We now demonstrate how the ``Gradient Mixing Error" can be used to guide the choice of the arrangement of a given set of nodes over a graph. 
Given a set of nodes having fixed data distributions, the parameters controlling the Gradient Mixing Error (GME) is controlled by the choice of mixing weights as well as the graph topology.
To illustrate the effects of the choice of topology on the convergence rates, we consider a toy setup of 4 nodes, having data distributions defined by quadratic objectives as in Section \ref{sec:quad}. 
\begin{figure*}
\centering
\hfill
\begin{subfigure}[t]{0.4\columnwidth}
    \includegraphics[width=\textwidth]{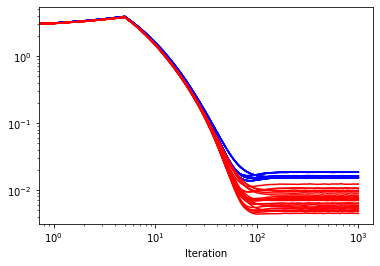}
    \caption{Red plots denote HADSGD while the blue plots denote mixing using the matrix corresponding to the optimal spectral gap.}
\end{subfigure}
\hfill
\begin{subfigure}[t]{0.4\columnwidth}
    \includegraphics[width=\textwidth]{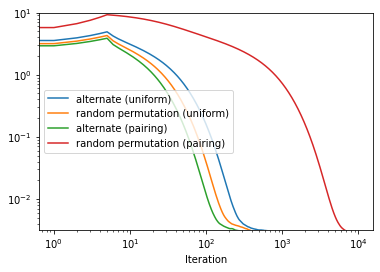}
    \caption{The colors denote combinations of different permutations and mixing matrices.}
    \label{fig:my_label}
\end{subfigure}
\hfill\null
\caption{Comparison of the distance from optimum vs number of iterations for different permutations of the nodes for (a) A random connected graph with 4 nodes (b) two-class ring topology setting with 16 nodes} 
\end{figure*}

To further illustrate the benefits of selecting an appropriate permutation, we consider a setup of 16 nodes distributed on a ring topology with the data distributions of exactly half of the nodes belonging to each one of the following class of objectives:
\begin{align*}
    f_1(\xx) &= \norm{A(\xx-\1)}^2\\
    f_2(\xx) &= \norm{A(\xx+\1)}^2,
\end{align*}
where $A$ denotes a fixed matrix with entries from $\mathcal{N}(0,1)$. We simulate the noise in SGD, by adding random vectors $\xi^{(t)} \sim \mathcal{N}(0,0.001)$ to the gradient updates for each node 
We compare the performance of DSGD under the following two permutations and choices of the mixing matrices:
\begin{enumerate}
    \item Heterogenous pairing: As illustrated in Figure \ref{fig:ring}, the nodes are ordered around the ring alternating between the data for objectives $f_1$ and $f_2$. Subsequently, every node is paired with exactly one of its neighbours such that the mixing steps involve averaging between the members of the pairs with equal weights of $0.5$. 
    \item Random permutation: The nodes are randomly distributed on the ring with the mixing matrix corresponding to the maximal spectral gap.
\end{enumerate}

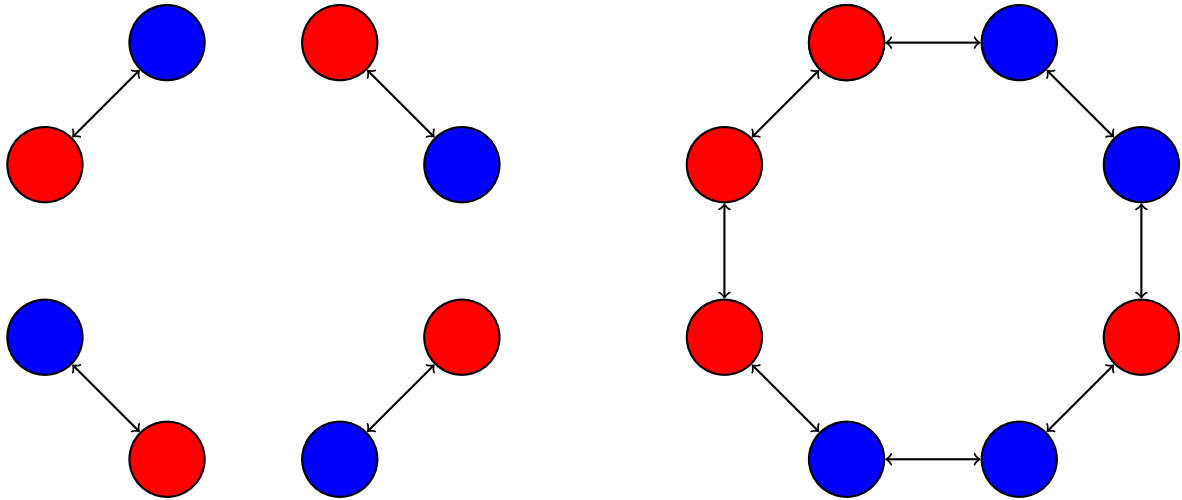
\begin{figure*}
    \centering
    \begin{subfigure}[b]{0.4\columnwidth}
    \begin{tikzpicture}[scale=1]
  \foreach \x in {0,2,4,6}
    \node[vertex,fill=blue] (\x) at ((22.5+\x*360/8:3) {};
\foreach \x in {1,3,5,7}
    \node[vertex,fill=red] (\x) at ((22.5+\x*360/8:3) {};
  \foreach \x/\y in {0/1,2/3,4/5,6/7}
    \draw[edge] (\x) to (\y);
\end{tikzpicture}
 \end{subfigure}
 \hspace{2cm}
 \begin{subfigure}[b]{0.4\columnwidth}
    \begin{tikzpicture}[scale=1]
  \foreach \x in {0,1,2,5,6}
    \node[vertex,fill=blue] (\x) at ((22.5+\x*360/8:3) {};
\foreach \x in {2,3,4,7}
    \node[vertex,fill=red] (\x) at ((22.5+\x*360/8:3) {};
  \foreach \x/\y in {0/1,1/2,2/3,3/4,4/5,5/6,6/7,7/0}
    \draw[edge] (\x) to (\y);
\end{tikzpicture}
 \end{subfigure}
 \caption{Different arrangements of data and mixing weights across a ring topology: (left) Heterogenous pairing between adjacent nodes having different data distributions, (right) Random permutation of nodes with uniform weights. The colors red and blue indicate two different classes of data distributions.}
 \label{fig:ring}
\end{figure*}

\section{Architectures and Hyperparameters}\label{app:arch}

\begin{table}[h]
	\caption{Experimental settings for Cifar-10}
	\footnotesize%
	\label{tab:cifar-experimental-settings}%
	\begin{tabularx}{\linewidth}{lX}
		\toprule
		Dataset              & Cifar-10                                                                                           \\
		Data augmentation    & random horizontal flip and random $32\times 32$ cropping                                                           \\
		Architecture         & Resnet20 with evonorm                                                                          \\
		Training objective   & cross entropy                                                                                                      \\
		Evaluation objective & top-1 accuracy                                                                                                     \\
		\midrule
		Number of workers    & 16, 32                                                                                                                 \\
		Topology             & Ring, Torus, Social Network\\
		Gossip weights       & Metropolis-Hastings (1/3 for ring)                                                                                 \\
		Data distribution    & Heterogeneous, not shuffled, according to Dirichlet sampling procedure from \cite{Lin2021:quasiglobal}              \\
		\midrule
		Batch size           & 32 patches per worker                                                                                              \\
		Momentum             & 0.9 (Nesterov)                                                                                                     \\
		Learning rate        & 0.1 for $\alpha=0.1$                                                                               \\
		LR decay             & $/10$ at epoch 150 and 180                                                                                         \\
		LR warmup            & Step-wise linearly within 5 epochs, starting from 0                                                                \\
		\# Epochs            & 300                                                                                                                \\
		Weight decay         & $10^{-4}$                                                                                                          \\
		Normalization scheme & no normalization layer                                                                                             \\
		\midrule
		Repetitions          & 3, with varying seeds                                                                                              \\
		\bottomrule
	\end{tabularx}
\end{table}

\begin{table}[h]
	\caption{Experimental settings for finetuning distilBERT}
	\footnotesize%
	\label{tab:bert-experimental-settings}%
	\begin{tabularx}{\linewidth}{lX}
		\toprule
		Dataset              & AG News                                                                    \\
		Data augmentation    & none                                                                                                  \\
		Architecture         & DistilBERT                                                                 \\
		Training objective   & cross entropy                                                                                         \\
		Evaluation objective & top-1 accuracy                                                                                        \\
		\midrule
		Number of workers    & 16                                                                                                    \\
		Topology             & ring                                                      \\
		Gossip weights       & Metropolis-Hastings (1/3 for ring)                                                                    \\
		Data distribution    & Heterogeneous, not shuffled, according to Dirichlet sampling procedure from \cite{Lin2021:quasiglobal} \\
		\midrule
		Batch size           & 32 patches per worker                                                                                 \\
		Adam $\beta_1$       & 0.9                                                                                                   \\
		Adam $\beta_2$       & 0.999                                                                                                 \\
		Adam $\epsilon$      & $10^{-8}$                                                                                             \\
		Learning rate        & 1e-6                                                                \\
		LR decay             & constant learning rate                                                                                \\
		LR warmup            & no warmup                                                                                             \\
		\# Epochs            & 10                                                                                                  \\
		Weight decay         & $0$                                                                                                   \\
		Normalization layer  & LayerNorm \citep{ba2016layer},                                                                     \\
		\midrule
		Repetitions          & 3, with varying seeds                                                                                 \\
		\bottomrule
	\end{tabularx}
\end{table}
\begin{table}[h]
	\caption{Experimental settings for ImageNet}
	\footnotesize%
	\label{tab:imagenet-experimental-settings}%
	\begin{tabularx}{\linewidth}{lX}
		\toprule
		Dataset              & ImageNet                                                                                  \\
		Data augmentation    & random resized crop ($224 \times 224$), random horizontal flip                                                     \\
		Architecture         & ResNet-20-EvoNorm~\citep{evonorm,he2015deep}                                                        \\
		Training objective   & cross entropy                                                                                                      \\
		Evaluation objective & top-1 accuracy                                                                                                     \\
		\midrule
		Number of workers    & 16                                                                                                                 \\
		Topology             & Ring\\
		Gossip weights       & Metropolis-Hastings (1/3 for ring)                                                                                 \\
		Data distribution    & Heterogeneous, not shuffled, according to Dirichlet sampling procedure from \cite{Lin2021:quasiglobal}              \\
		\midrule
		Batch size           & 32 patches per worker                                                                                              \\
		Momentum             & 0.9 (Nesterov)                                                                                                     \\
		Learning rate        & $0.1 \times \frac{32 * 16}{256}$)                                         \\
		LR decay             & $/10$ at epoch $30, 60, 80$                                                                                        \\
		LR warmup            & Step-wise linearly within 5 epochs, starting from 0.1                                                              \\
		\# Epochs            & 90                                                                                                                 \\
		Weight decay         & $10^{-4}$                                                                                                          \\
		Normalization layer  & EvoNorm~\citep{evonorm}                                                                                     \\
		\bottomrule
	\end{tabularx}
\end{table}

\end{document}

%% file: math_commands.tex
\usepackage{amsmath,amsfonts,bm}

\def\eqref#1{equation~(\ref{#1})}

\def\1{\bm{1}}

\DeclareMathAlphabet{\mathsfit}{\encodingdefault}{\sfdefault}{m}{sl}
\SetMathAlphabet{\mathsfit}{bold}{\encodingdefault}{\sfdefault}{bx}{n}

\newcommand{\E}{\mathbb{E}}

\newcommand{\R}{\mathbb{R}}

\DeclareMathOperator*{\argmin}{arg\,min}

\DeclareMathOperator{\Tr}{Tr}

\renewcommand{\epsilon}{\varepsilon}

%% file: macros.tex
\usepackage{amsfonts}
\usepackage{amsmath}
\usepackage{amsthm}
\usepackage{amssymb}
\usepackage{dsfont}

\usepackage{xcolor}
\usepackage{color}
\usepackage{graphicx}

\usepackage[algo2e]{algorithm2e} 
\usepackage{verbatim}
\usepackage{xspace} %

\usepackage{enumerate}
\usepackage{enumitem}

\providecommand{\lin}[1]{\ensuremath{\left\langle #1 \right\rangle}}
\providecommand{\abs}[1]{\left\lvert#1\right\rvert}
\providecommand{\norm}[1]{\left\lVert#1\right\rVert}

  \providecommand{\R}{\mathbb{R}} %
  \providecommand{\EE}[2]{{\mathbb E}_{#1}\left.#2\right. }  %

  \DeclareMathOperator*{\supp}{supp}

  \providecommand{\1}{\mathbf{1}}
  \renewcommand{\aa}{\mathbf{a}}
  \providecommand{\bb}{\mathbf{b}}

  \providecommand{\uu}{\mathbf{u}}

  \providecommand{\xx}{\mathbf{x}}
  \providecommand{\yy}{\mathbf{y}}

  \providecommand{\cD}{\mathcal{D}}

  \providecommand{\cG}{\mathcal{G}}

  \providecommand{\cM}{\mathcal{M}}
  \providecommand{\cN}{\mathcal{N}}
  \providecommand{\cO}{\mathcal{O}}

  \providecommand{\cW}{\mathcal{W}}

  \usepackage{bm}

  \usepackage[textwidth=5cm]{todonotes}
  
\providecommand{\mycomment}[3]{\todo[caption={},size=footnotesize,color=#1!20]{\textbf{#2: }#3}}%
\providecommand{\inlinecomment}[3]{%
  {\color{#1}#2: #3}}%
\newcommand\commenter[2]%
{%
  \expandafter\newcommand\csname i#1\endcsname[1]{\inlinecomment{#2}{#1}{##1}}
  \expandafter\newcommand\csname #1\endcsname[1]{\mycomment{#2}{#1}{##1}}
}

\newtheorem{proposition}{Proposition}
\newtheorem{lemma}{Lemma}

\newtheorem{remark}{Remark}
\newtheorem{assumption}{Assumption}
\newtheorem{theorem}{Theorem}
\newtheorem{example}{Example}

  \definecolor{mydarkblue}{rgb}{0,0.08,0.45}
  \usepackage{hyperref}
  \hypersetup{ %
    pdftitle={},
    pdfauthor={},
    pdfsubject={},
    pdfkeywords={},
    pdfborder=0 0 0,
    pdfpagemode=UseNone,
    colorlinks=true,
    linkcolor=mydarkblue,
    citecolor=mydarkblue,
    filecolor=mydarkblue,
    urlcolor=mydarkblue,
    pdfview=FitH}
\usepackage[capitalize,noabbrev]{cleveref}

\usepackage{url}